\documentclass[letterpaper]{article} % DO NOT CHANGE THIS
\usepackage{amsmath,amsfonts,bm}

\def\eqref#1{Eq.~\ref{#1}}
\def\1{\bm{1}}
\def\rvc{{\mathbf{c}}}

\def\rvx{{\mathbf{x}}}

\def\meps{{\bm{\epsilon}}}

\DeclareMathAlphabet{\mathsfit}{\encodingdefault}{\sfdefault}{m}{sl}
\SetMathAlphabet{\mathsfit}{bold}{\encodingdefault}{\sfdefault}{bx}{n}
\def\gD{{\mathcal{D}}}

\def\gL{{\mathcal{L}}}

\def\gN{{\mathcal{N}}}

\usepackage{aaai25}  % DO NOT CHANGE THIS
\usepackage{times}  % DO NOT CHANGE THIS
\usepackage{helvet}  % DO NOT CHANGE THIS
\usepackage{courier}  % DO NOT CHANGE THIS
\usepackage[hyphens]{url}  % DO NOT CHANGE THIS
\usepackage{graphicx} % DO NOT CHANGE THIS
\usepackage{natbib}  % DO NOT CHANGE THIS AND DO NOT ADD ANY OPTIONS TO IT
\usepackage{caption} % DO NOT CHANGE THIS AND DO NOT ADD ANY OPTIONS TO IT
\usepackage{graphicx} 
\usepackage{subcaption}
\usepackage{enumitem}
\usepackage{parskip}
\usepackage{multirow}
\usepackage{adjustbox}
\usepackage{moresize}
\usepackage{lipsum}
\usepackage[linesnumbered,ruled,vlined]{algorithm2e}

\SetCommentSty{mycommfont}
\SetKwInput{KwInput}{Input}
\SetKwInput{KwOutput}{Output}
\usepackage{amsmath}

\DeclareMathOperator*{\argmin}{arg\,min}
\usepackage{amssymb}
\usepackage{mathtools}
\usepackage{amsthm}

\newtheorem{theorem}{Theorem}[section]

\newtheorem{lemma}[theorem]{Lemma}

\theoremstyle{definition}
\newtheorem{definition}[theorem]{Definition}

\theoremstyle{remark}

\usepackage{newfloat}
\usepackage{listings}
\DeclareCaptionStyle{ruled}{labelfont=normalfont,labelsep=colon,strut=off} % DO NOT CHANGE THIS
\title{Text2Data: Low-Resource Data Generation with Textual Control}
\author{
    Shiyu Wang,
    Yihao Feng,
    Tian Lan,
    Ning Yu,
    Yu Bai,
    Ran Xu,
    Huan Wang\thanks{Corresponding author: huan.wang@salesforce.com},
    Caiming Xiong,\\
    Silvio Savarese
}
\affiliations{
    Salesforce AI Research\\
    181 Lytton Ave, Palo Alto, CA 94301
}

\usepackage{bibentry}
\begin{document}

\maketitle
\vspace{-12mm}
\begin{abstract}
Natural language serves as a common and straightforward signal for humans to interact seamlessly with machines. Recognizing the importance of this interface, the machine learning community is investing considerable effort in generating data that is semantically coherent with textual instructions. While strides have been made in text-to-data generation spanning image editing, audio synthesis, video creation, and beyond, low-resource areas characterized by expensive annotations or complex data structures, such as molecules, motion dynamics, and time series, often lack textual labels. This deficiency impedes supervised learning, thereby constraining the application of advanced generative models for text-to-data tasks. In response to these challenges in the low-resource scenario, we propose Text2Data, a novel approach that utilizes unlabeled data to understand the underlying data distribution through an unsupervised diffusion model. Subsequently, it undergoes controllable finetuning via a novel constraint optimization-based learning objective that ensures controllability and effectively counteracts catastrophic forgetting. Comprehensive experiments demonstrate that Text2Data is able to achieve enhanced performance regarding controllability across various modalities, including molecules, motions and time series, when compared to existing baselines.
\textbf{Code}\textemdash https://github.com/SalesforceAIResearch/text2data
\textbf{Extended version}\textemdash https://arxiv.org/abs/2402.10941
\end{abstract}

\section{Introduction}
Autonomy and controllability stand as twin primary pillars of generative AI~\citep{gozalo2023chatgpt, wang2022controllable}. While the challenge of autonomy has been substantially addressed through the rapid advancements of generative models, controllability is now ascending as a fervently explored arena within the machine learning community. As natural languages are one of the most common and simplest control signal for human beings to interact with machines, the machine learning community has increasingly focused on generating data that aligns semantically with textual descriptions, given its wide-ranging applications such as image editing~\citep{zhang2023sine, kawar2023imagic}, audio synthesis~\citep{pmlr-v202-liu23f, huang2023make}, video generation~\citep{li2018video, hu2022make}, and many more~\citep{tevet2023human, sanghi2022clip}. 

Recent breakthroughs in text-to-data generative models, particularly those using diffusion ~\citep{li2023diffusion, yang2023diffsound, kumari2023multi}, have demonstrated remarkable proficiency by harnessing the rich semantic insights from vast datasets of data-text pairs. Despite the broad application of text-to-data generative models, not all modalities can meet the substantial data-text pair requirements for achieving optimal controllability during model training. This is often due to costly annotations or intricate data structures, a scenario we refer to as the low-resource situation. The lack of text labels in certain areas, such as molecules~\citep{ramakrishnan2014quantum, irwin2012zinc}, motions~\citep{guo2020action2motion, mahmood2019amass}, and time series~\citep{du2020multivariate}, primarily restricts supervised learning and hinders the use of advanced generative models for text-to-data generation tasks. The low-resource situation when training generative models unsurprisingly results in issues like undesirable generation quality, model overfitting, bias, and lack of diversity. However, the optimization for scarce text representations to improve the alignment between generated data and input texts in generative models is still under-explored.

To mitigate the issues in the low-resource scenario, strategies such as data augmentation~\citep{hedderich2020survey, meng2021mixspeech}, semi-supervised learning~\citep{thomas2013deep, cheuk2021reconvat}, and transfer learning~\citep{tits2020exploring, yi2018language} are utilized. Yet, each comes across challenges. Data augmentation, for example, cannot always replicate genuine data fidelity to align accurately with initial text descriptions, and potentially leads to overfitting due to over-reliance on augmented samples. It also exacerbates the training complexity, intensifying the already high computational demand of diffusion models. For semi-supervised learning, text inherently carries nuances, ambiguities, and multiple meanings. Ensuring that the model infers the correct interpretation when leveraging unlabeled data is not straightforward. Lastly, while transfer learning offers a solution for limited datasets, it is prone to catastrophic forgetting~\citep{iman2023review}, where previous knowledge diminishes as new text descriptions are introduced.

Alternative to existing solutions, we propose Text2Data, a diffusion-based framework achieving enhanced text-to-data controllability even under low-resource situation. Specially, Text2Data operates in two pivotal stages: \textbf{(1) Distribution mastery by leveraging unlabeled data}. Distinct from conventional semi-supervised learning methods, Text2Data does not aim to deduce labels for unlabeled data. Instead, this step uses unlabeled data to discern the overarching data distribution via an unsupervised diffusion model, eliminating the semantic ambiguity often associated with semi-supervised approaches. \textbf{(2) Controllable finetuning on text-labeled data}. The learned diffusion model is then finetuned by text-labeled data. Distinct from methods reliant on data augmentation, Text2Data abstains from inflating the training dataset. Instead, we introduce a novel constraint optimization-based learning objective, aiming to mitigate catastrophic forgetting by regularizing the model parameter space closely to its preliminary space before finetuning. Our contributions are summarized as follows:
\begin{itemize}[leftmargin=*]
    \item We introduce Text2Data, a novel framework designed for text-to-data generation in the low-resource scenario. This approach maintains the fine-grained data distribution by fully harnessing both labeled and unlabeled data.
    \item We design a novel learning objective based on constraint optimization to achieve controllability and overcome catastrophic forgetting during finetuning.
    \item We theoretically validate optimization constraint selection and generalization bounds for our learning objective.
    \item We compile real-world datasets across three modalities and conduct comprehensive experiments to show the effectiveness of Text2Data. The results demonstrate that Text2Data achieves superior performance baselines regarding both generation quality and controllability.
    % \vspace{-1mm}
\end{itemize}
\vspace{-2mm}
\section{Related works}
\label{sec:related_works}
\subsection{Text-to-data diffusion-based generation}
Diffusion models, notably divided into classifier-guided~\citep{dhariwal2021diffusion} and classifier-free~\citep{ho2022classifier} categories, have significantly impacted data generation across various domains.~\citep{hoogeboom2022equivariant, yang2023diffsound,ho2022imagen, voleti2022mcvd}. The classifier-guided diffusion guides the model during inference phase by independently training a classifier and supervising the model with its gradient, which is inefficient when computing gradient at each time step and sometimes the generation quality is deficient as the guidance is not involved in the training. By contrast, classifier-free diffusion guidance blends score estimates from both a conditional diffusion model and an unconditional one with time step as a parameter, exemplified by E(3) Equivariant Diffusion Model (EDM)~\citep{hoogeboom2022equivariant} and Motion Diffusion Model (MDM)~\citep{tevet2023human} for controllable molecule and motion generation, respectively. Furthermore, since natural languages are a prevalent medium for human to communicate with the world, the text-to-data generation paradigm has gained traction, with diffusion models being instrumental in generating high-quality data aligned with textual inputs. The extensive applications encompass text-to-image generation~\citep{ruiz2023dreambooth, zhang2023adding}, text-to-speech generation~\citep{huang2022prodiff, kim2022guided}, text-to-shape generation~\citep{li2023diffusion, lin2023magic3d}, and more, leveraging the abundant text descriptions for training potent generative models. Despite advancements in generating data from text across various modalities, many other modalities may not satisfy the stringent requirements for sufficient data-text pairs essential for attaining optimal controllability during the training of models.

% \ning{Try to consolidate the above two subsections into one because they both discuss about controllable diffusion models. In the end, discuss the relationship between this subsection and our work. Do we build on top of it, or how to differentiate our work from existing ones?}
\subsection{Low-resource learning}
In response to the challenges of low-resource training for controllable generative models, several strategies have been formulated. For instance, \citet{yin2023ttida} proposes Text-to-Text-to-Image Data Augmentation that employed both large-scale pretrained Text-to-Text and Text-to-Image generative models for data augmentation to generate photo-realistic labeled images in a controllable manner. \citet{zang2019semi} utilizes a semi-supervised approach to augment the training of both the encoder and decoder with both labeled and unlabeled data. \citet{tu2019end} proposes to learn a mapping between source and target linguistic symbols and employs transfer learning to transfer knowledge from a high-resource language to low-resource language. Nevertheless, all those strategies have their own limitations, such as high computational complexity for data augmentation, difficulty in maintaining the correct interpretation of text when leveraging unlabeled data during semi-supervised learning, and the potential catastrophic forgetting issues in transfer learning. Additionally, most of works on text-to-data generation lie in the modality of image, speeches, and texts, which have plenty of labeled datasets~\citep{ljspeech17, Pratap2020MLSAL, lin2014microsoft, jiang2021talkedit, wang2020covost} to train the models nowadays. Yet it is far under-explored for the low-source modalities such as molecules, motions and time series. Therefore, we propose Text2Data, a diffusion-based framework adept at harnessing limited text-labeled data to enhance the controllability of the model in text-to-data generation. 
\vspace{-2mm}
\section{Problem formulation}
\label{sec:problem_formulation}
Suppose the dataset $\gD=\{\rvx, \rvc\}$ contains $N$ independent samples in total, where $\rvx=\{\rvx_i\}_{i=1}^N$ is the data samples such as molecules, motions, time series, etc. We assume that there is only a proportion of data in $\rvx$ that has corresponding text description $\rvc=\{\rvc_i\}_{i=1}^{N_p}$ where $N_p\le N$. We denote that data with text description is contained in $\gD_p$ and $\gD_p\subset\gD$. Using both text-labeled and unlabeled data in $\gD$, we aim to learn a generative model, $p_\theta(\rvx\vert \rvc)$ parameterized by $\theta$ that is able to generate data $\rvx\sim p_\theta(\rvx\vert \rvc)$ corresponding to specific text description $\rvc = \rvc^*$.
\begin{figure}[h]
\begin{center}
\includegraphics[width=0.45\textwidth]{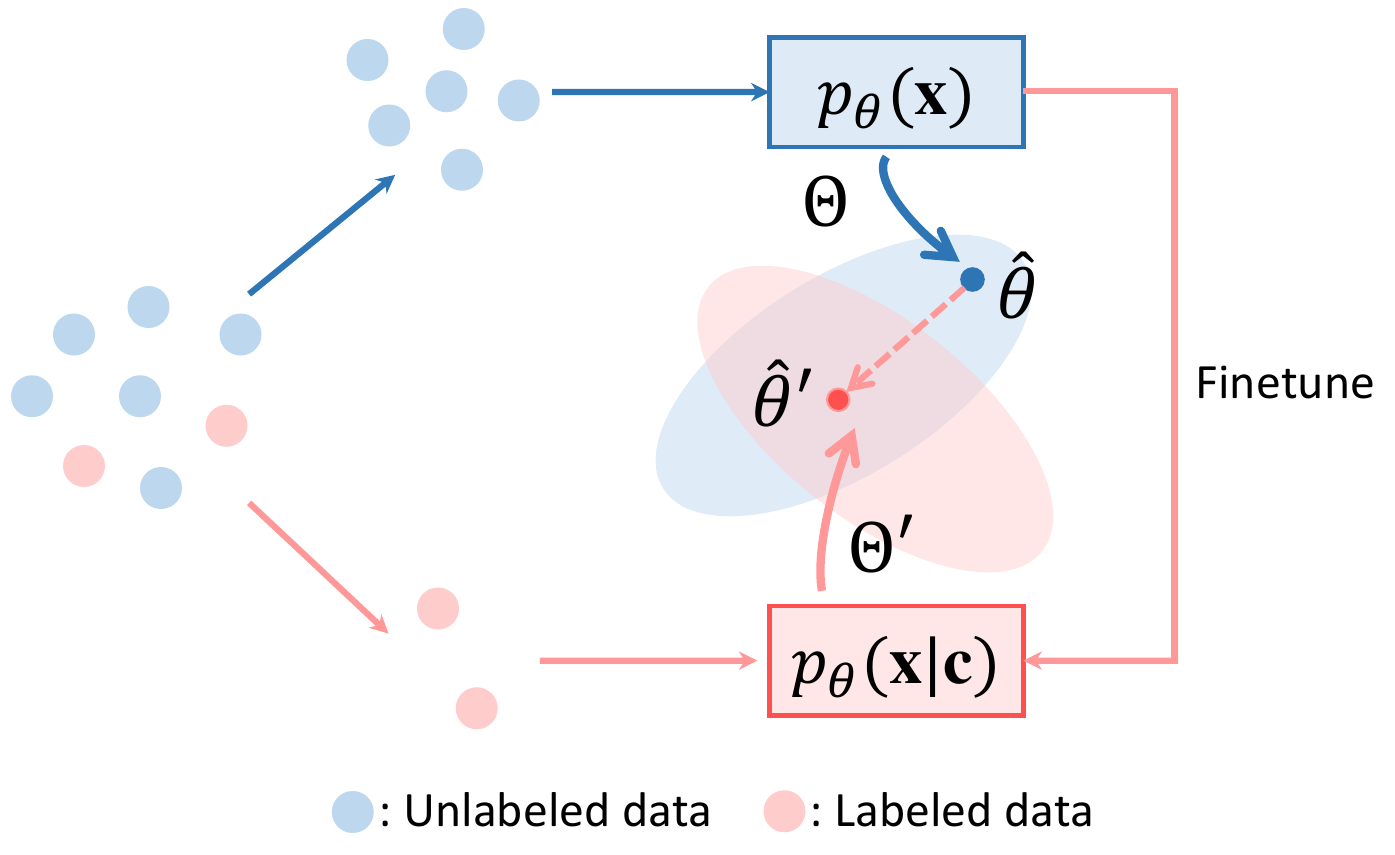}
\caption{Overview of Text2Data. The model leverages unlabeled data (i.e., blue module) to discern the overall data distribution while the optimal set of model parameters $\Theta$ is obtained. Then the model is finetuned on labeled data (i.e., red module) by constraint optimization that gives the optimal set of parameters as $\Theta\cap\Theta'$, where $\Theta'$ is the optimal set of parameters if finetune the model without constraint.}
\label{fig:model}
\end{center}
\vspace{-2mm}
\end{figure}
\vspace{-2mm}
\section{Methods}
\label{sec:methods}
% In this section, we unpack the intricacies of Text2Data. Section~\ref{sec:step1} covers how we utilize all unlabeled data to grasp the overarching data distribution. Section~\ref{sec:step2} elaborates on achieving controllability with a limited subset of text-labeled data. Finally, Section~\ref{sec:bound} offers theoretical insights into the selection of the optimization constraint, underpinned by generalization bounds for our learning objective.

Controllable data generation seeks to learn the conditional data distribution $p_\theta(\rvx\vert \rvc)$ during training and subsequently draw samples from this assimilated distribution during the inference stage. Consequently, our primary objective during the training phase is to optimize the following:
\begin{align}
    \min_\theta~\mathbb E_{\rvx, \rvc\sim p_{\gD_p}(\rvx, \rvc)}[-\log p_\theta(\rvx\vert\rvc)]\,.
    \label{eq:primary_l}
\end{align}
While the training of such generative models is contingent upon the supervision of text descriptions present in the dataset (i.e., $\gD_p$), it is not always feasible to obtain an adequate number of data-text pairs to ensure optimal controllability (i.e., $\vert\gD_p\vert<\vert\gD\vert$), especially in modalities like molecular structures, motion patterns and time series. Such constraints can precipitate complications, including model overfitting when optimizing according to Eq. \ref{eq:primary_l}. Given the challenges, devising strategies to capitalize on the unlabeled data within $\gD$—which is often more accessible and cost-effective—is pivotal to effectively learn $p_\theta(\rvx\vert\rvc)$. 

Notably, the marginal distributions learned from unlabeled data closely resemble those obtained from labeled data:
\small
\begin{align}
    p_\theta(\rvx) \approx \int p_\theta(\rvx\vert\rvc) p_{\gD_{p}}(\rvc) d\rvc = \mathbb E_{\rvc\sim p_{\gD_p}(\rvc)}[ p_\theta(\rvx\vert \rvc)] ,
    \label{eq:approx}
\end{align}
\normalsize
where $p_{\gD_p}(\rvc)$ is the true underlying text generating distribution corresponding to $\rvx\in\gD_p$. Hence, as per Figure~\ref{fig:model}, we are inspired to initially utilize the unlabeled data in $\gD$ to learn $p_\theta(\rvx)$ and obtain the optimal set of model parameters $\hat{\theta}\in\Theta$, which serves as a robust approximation of $p_\theta(\rvx\vert\rvc)$. Then, we finetune it using data-text pairs in $\gD_p$ to achieve desired model control. Crucially, we ensure that the parameters remain in close proximity to those established when learning $p_\theta(\rvx)$ via constraint optimization to make Eq. \ref{eq:approx} hold while mitigating catastrophic forgetting. Figure~\ref{fig:model} shows our constraint to keep the finetuned parameters $\hat{\theta}'$ within $\Theta\cap\Theta'$, where $\Theta'$ represents the optimal parameters from unconstrained finetuning. Then we plug the generative diffusion implementation to our framework. Finally, we use the last section to offer theoretical insights into the selection of the optimization constraint, underpinned by generalization bounds for our learning objective.

\subsection{Learning text-to-data generation under low-resource situation}
\label{sec:overall}
To capture a general distribution of data, we leverage all data in $\gD$ incorporating NULL tokens as conditions (i.e., $\rvc=$ NULL) to facilitate subsequent finetuning. Specifically, we train the generative model $p_{\theta}(\rvx\vert\emptyset)$, where $\theta$ parameterizes the model and $\emptyset$ represents the NULL token in practice. As the NULL token is independent to $\rvx$, we also have $p_\theta(\rvx\vert\emptyset) = p_\theta(\rvx)$. Hence, we are equivalently optimizing the following:
\small
\begin{align}
    \min_\theta~\mathbb E_{\rvx\sim p_\gD(\rvx)}[-\log p_{\theta}(\rvx)],
    \label{eq:obj_l1}
\end{align}
\normalsize
where $p_\gD(\rvx)$ is the true underlying data generating distribution. Having discerned the general data distribution through Eq. \ref{eq:obj_l1}, we proceed to finetune $p_\theta(\rvx\vert \rvc)$ using text labels $\rvc$ in $\gD_p$ to achieve the controllability of model. In alignment with Eq. \ref{eq:approx}, the anticipated finetuned parameter should approximate the parameter optimized in Eq. \ref{eq:obj_l1}. We achieve this by finetuning $p_\theta(\rvx\vert \rvc)$ using the labeled data in $\gD_p$ within the optimal set obtained in Eq.~\ref{eq:obj_l1}, leading to the subsequent learning objective for the finetuning phase:
\small
\begin{align}
    \min_{\theta}~&\mathbb E_{\rvx, \rvc\sim p_{\gD_{p}}(\rvx, \rvc)}[-\log p_\theta(\rvx\vert\rvc)]\nonumber \\
    \text{s.t.}~&\mathbb E_{\rvx\sim p_{\gD_p}(\rvx)}[-\log p_\theta(\rvx)] \le \xi,\nonumber\\
    ~&\xi=\inf_{\theta\in\Theta} \mathbb E_{\rvx\sim p_{\gD}(\rvx)}[-\log p_\theta(\rvx)]
    \label{eq:l2}
\end{align}
\normalsize
where $p_{\gD_{p}}(\rvx, \rvc)$ is the true underlying data-text joint distribution. $\Theta$ denotes a localized parameter space where a minimum can be located. Specifically, we minimize $\mathbb E_{\rvx, \rvc\sim p_{\gD_{p}}(\rvx, \rvc)}[-\log p_\theta(\rvx\vert\rvc)]$ using the labeled data in $\gD_p$ within the optimal set $\{\theta: \mathbb E_{\rvx\sim p_{\gD_p}(\rvx)}[-\log p_\theta(\rvx)]\le \xi\}$ to make the parameter not far from learned via Eq. \ref{eq:obj_l1}, so that catastrophic forgetting is mitigated. The trade-off between the learning objective and its constraint aligns with the nature of a \textit{lexicographic optimization} problem~\citep{gong2021bi} and therefore can be solved in that context.

% As $p_{\gD_p}(\rvx, \rvc)$ is unknown so that we compute the empirical loss of Eq. \ref{eq:obj_l2}, which leads to the following empirical learning objective:
% \begin{align}
%     \min_{\theta}~&\mathbb E_{\rvx,\rvc\sim \hat{p}_{\gD_p}(\rvx\vert\rvc)}[-\log p_\theta(\rvx,\rvc)]\nonumber \\
%     \text{s.t.}~&\mathbb E_{\rvx\sim \hat{p}_{\gD_p}(\rvx)}[-\log p_\theta(\rvx)] \le \hat{\xi},~\text{where}~\hat{\xi}=\inf_{\theta\in\Theta} \mathbb E_{\rvx\sim \hat{p}_{\gD}(\rvx)}[-\log p_\theta(\rvx)]
%     \label{eq:emp_l2}
% \end{align}

\subsection{Generative objective on empirical samples}
\label{sec:imp}

\eqref{eq:l2} is the population-level learning objective while empirically we follow the standard loss function (i.e., transformed evidence lower bound of Eq. \ref{eq:l2}) of the classifier-free diffusion guidance~\citep{ho2022classifier} by optimizing:
\small
\begin{align}
\min_\theta~&\gL_2(\theta)~~~~\text{s.t.}~\gL_1'(\theta)\le\xi,~\xi=\inf_{\theta\in\Theta}\gL_1(\theta),
\label{eq:obj_l2}
\end{align}
\normalsize
where we have: $\gL_1(\theta) = \mathbb E_{\rvx\sim p_{\gD}(\rvx), t}[\vert\vert\meps_{\theta}(\rvx^{(t)}, t) - \meps\vert\vert^2]$, $\gL_1'(\theta) = \mathbb E_{\rvx\sim p_{\gD_p}(\rvx), t}[\vert\vert\meps_{\theta}(\rvx^{(t)}, t) - \meps\vert\vert^2]$ and $\gL_2(\theta) = \mathbb E_{\rvx, \rvc\sim p_{\gD_p}(\rvx, \rvc), t}[\vert\vert\meps_{\theta}(\rvx^{(t)}, \rvc, t) - \meps\vert\vert^2]$. Also, $t$ is sampled from uniform between 1 and $T$, $T$ is the total diffusion steps, $\meps$ is the standard Gaussian random variable, and $\meps_\theta(\rvx^{(t)}, t)$ and $\meps_\theta(\rvx^{(t)}, \rvc, t)$ are functions we aim to fit at the $t$-th diffusion step. Note that $\meps_\theta(\rvx^{(t)}, t)$ and $\meps_\theta(\rvx^{(t)}, \rvc, t)$ share the same parameters but are just trained at different stages: distribution mastery on unlabeled data and controllable finetuning on labeled data, respectively. The framework of classifier-free diffusion model and the derivation of $\gL_1(\theta)$, $\gL_1'(\theta)$ and $\gL_2(\theta)$ are introduced in Appendix.

As the true data generating distributions $p_{\gD}(\rvx)$, $p_{\gD_p}(\rvx)$ and $p_{\gD_p}(\rvx, \rvc)$ are unknown, we instead optimize the following empirical loss:
\vspace{-2mm}
\small
\begin{align}
\min_\theta~\hat{\gL}_2(\theta)~~~~\text{s.t.}~\hat{\gL}_1'(\theta)\le\hat{\xi},~\hat{\xi}=\inf_{\theta\in\hat{\Theta}}\hat{\gL}_1(\theta),
\label{eq:emp_l2}
\end{align}
\normalsize
where we have $\hat{\gL}_1(\theta) = \mathbb E_{\rvx\sim \hat{p}_{\gD}(\rvx), t}[\vert\vert\meps_{\theta}(\rvx^{(t)}, t) - \meps\vert\vert^2]$, $\hat{\gL}_1'(\theta) = \mathbb E_{\rvx\sim \hat{p}_{\gD_p}(\rvx), t}[\vert\vert\meps_{\theta}(\rvx^{(t)}, t) - \meps\vert\vert^2]$ and $\hat{\gL}_2(\theta) = \mathbb E_{\rvx, \rvc\sim \hat{p}_{\gD_p}(\rvx, \rvc), t}[\vert\vert\meps_{\theta}(\rvx^{(t)}, \rvc, t) - \meps\vert\vert^2]$. $\hat{p}_{\gD}(\rvx)$, $\hat{p}_{\gD_p}(\rvx)$ and $\hat{p}_{\gD_p}(\rvx, \rvc)$ are corresponding empirical data generating distributions. $\hat{\Theta}$ is the localized parameter space where a minimum can be located for $\hat{\gL}_1(\theta)$. The lexicographic optimization of Eq.~\ref{eq:emp_l2} is presented in Algorithm~\ref{alg:learn}, which relies on learning the update direction on model parameter (i.e., $\nabla \hat\gL_2(\theta) + \lambda\nabla\hat{\gL}_1'(\theta)$ in Algorithm~\ref{alg:learn}) to balance the trade-off between the minimization of $\hat{\gL}_2(\theta)$ and $\hat{\gL}_1'(\theta)$ using a dynamic gradient descent~\citep{gong2021bi}:
\small
\begin{align}
    \theta\leftarrow\theta - \omega\cdot(\nabla \hat\gL_2(\theta) + \lambda\nabla\hat{\gL}_1'(\theta)),
\end{align}
\normalsize
where $\omega$ is predefined positive step size, and $\lambda$ is calculated based on whether the constraint is satisfied at the current gradient step:
\small
\begin{align}
    \lambda =& \max(\frac{\phi(\theta) - \nabla\hat\gL_2(\theta)^T\nabla\hat{\gL}_1'(\theta)}{\|\nabla\hat{\gL}_1'(\theta)\|^2}, 0)\nonumber \\
    \phi(\theta) =&\min(\alpha(\hat{\gL}_1'(\theta)-\gamma\cdot\hat\xi), \beta\|\nabla\hat{\gL}_1'(\theta)\|^2),
\end{align}
\normalsize
where $\alpha$, $\beta$ and $\gamma$ are predefined positive hyperparameters. More details and intuition of lexicographic optimization are explained in Appendix.

\begin{algorithm}
\scriptsize
\DontPrintSemicolon
\KwInput{$\hat{\xi}=\inf_{\theta\in\hat{\Theta}}\hat{\gL}_1(\theta)$ by pretraining on $\gD$; Total diffusion steps $T$; Scheduled forward variance $\{\beta_t\}_{t=1}^T$; $\{\alpha_t=1-\beta_t\}_{t=1}^T$; Predefined positive hyperparameters $\alpha$, $\beta$, $\gamma$ and $\omega$; probability of unconditional training $p_{uncond}$} 
\While{Not converged}
   {
   ~~~Sample $\rvx, \rvc\sim \hat{p}_{\gD_p}(\rvx, \rvc)$, $\meps\sim\gN(\bm 0, \bm I)$, $t\sim U(1, T)$\\
   Compute $\bar\alpha_t=\prod_{s=1}^t\alpha_t$\\
   Diffuse $\rvx^{(t)}=\sqrt{\bar\alpha_t}\rvx + \sqrt{1-\bar\alpha_t}\meps$\\
   Replace $\rvc$ with $\emptyset$ with probability $p_{uncond}$\\
   Compute $\hat{\gL_2}(\theta)=\|\meps_{\theta}(\rvx^{(t)}, t)-\meps\|^2$, $\hat\gL_1'(\theta)=\|\meps_{\theta}(\rvx^{(t)}, \rvc, t)-\meps\|^2$\\
   Compute $\phi(\theta) = \min(\alpha(\hat{\gL}_1'(\theta)-\gamma\cdot\hat\xi), \beta\|\nabla\hat{\gL}_1'(\theta)\|^2)$\\
   Compute $\lambda = \max(\frac{\phi(\theta) - \nabla\hat\gL_2(\theta)^T\nabla\hat{\gL}_1'(\theta)}{\|\nabla\hat{\gL}_1'(\theta)\|^2}, 0)$\\
   Update $\theta$ by $\theta - \omega\cdot(\nabla \hat\gL_2(\theta) + \lambda\nabla\hat{\gL}_1'(\theta))$
   }
\caption{Lexicographic optimization on Eq.~\ref{eq:l2}}
\label{alg:learn}
\end{algorithm}
\vspace{-3mm}
Furthermore, the lexicographic optimization-based constraint $\hat{\xi}=\inf_{\theta\in\hat{\Theta}}\hat{\gL}_1(\theta)$ in Eq. \ref{eq:emp_l2} may be overly strict and could require relaxation to ease the training process. While we anticipate that the parameters derived from Eq. \ref{eq:emp_l2} should be close to those from Eq. \ref{eq:obj_l2}, they do not necessarily have to be an exact subset of the parameters from Eq. \ref{eq:obj_l2}. 
\subsection{Generalization bound of learning constraint}
\label{sec:bound}
In this section, we deduce a confidence bound for the constraint to demonstrate that the optimal set of minimizing the empirical loss within the derived confidence bound encapsulates the true one, and guide the further relaxation on the constraint if needed. First, we define sub-Gaussian random variable in Definition~\ref{def:subgaussian} followed by the deducted confidence bound in Theorem~\ref{thm:bound}. 
% \vspace{-2mm}
\begin{definition}[Sub-Gaussian random variable]
    The random variable $X$ with mean 0 is sub-Gaussian with variance $\sigma^2$ if $\forall s\in\mathbb R$, $\mathbb E_X[\exp (sX)]\le \exp(\frac{\sigma^2 s^2}{2})$.
    \label{def:subgaussian}
\end{definition}
% \vspace{-2mm}
Based on the fact that $\rvx^{(t)}$ is diffused from the random variable $\rvx$ and the standard Gaussian noise $\meps$ (i.e., Algorithm~\ref{alg:learn}), therefore, $\meps_\theta(\rvx^{(t)}, t)$ and $\meps_\theta(\rvx^{(t)}, \rvc^{(t)}, t)$ are also random variables. Meanwhile, minimizing the loss in \eqref{eq:obj_l2} pushes $\meps_\theta(\rvx^{(t)}, t)$ and $\meps_\theta(\rvx^{(t)}, \rvc^{(t)}, t)$ towards $\meps$, which has the mean of 0. Then we introduce the following theorem 1 while assuming $\meps_\theta(\rvx^{(t)}, t)$ and $\meps_\theta(\rvx^{(t)}, \rvc^{(t)}, t)$ are sub-Gaussian random variables.
% \vspace{-2mm}
\begin{theorem}
For every $\theta$ and $t$, assume $\meps_\theta(\rvx^{(t)}, t)$ and $\meps_\theta(\rvx^{(t)}, \rvc_i, t)$ are sub-Gaussian random variables with mean 0 and variance $\sigma^2$, and $\Theta$ is finite. Let $\Theta^*=\{\theta:\gL_1'(\theta)\le\xi\}$, $\hat{\Theta}^*=\{\theta:\hat{\gL}_1'(\theta)\le\hat{\xi} + \epsilon\}$ where $\epsilon$ is the confidence bound with the probability of $1-\delta$. Let $\theta^*$ be the solution to Eq. \ref{eq:obj_l2} and $\hat{\theta}^*$ be the solution to empirical Eq. \ref{eq:emp_l2}, then we have the following:
% \yub{Let $\hat{\theta}$ denote the solution to empirical problem XXX.}
\begin{enumerate}[leftmargin=2em]
  \item $\Theta^*\subseteq\hat{\Theta}^*$: the set of $\theta$ by optimizing Eq. \ref{eq:emp_l2} within the confidence bound covers the true one.
  \item $\gL_2(\hat{\theta}^*)\le \gL_2(\theta^*) + 2\epsilon_{N_p}$: $\theta^*$ and $\hat{\theta}^*$ compete well on $\gL_2(\theta)$.
  \item $\gL_1'(\hat{\theta}^*)\le\xi + 2\epsilon_{N_p} + 2\epsilon_N$: $\hat{\theta}^*$ does not violate constraint of Eq. \ref{eq:obj_l2} too much on $\gL_1'(\theta)$.
  % ~\yub{first $\epsilon$ should be $\xi$?}
  \label{thm:bound}
\end{enumerate}
\vspace{-2mm}
where $\epsilon=\epsilon_N + \epsilon_{N_p}$, $\epsilon_N = \sqrt{C\tilde{\sigma}^2}\cdot\sqrt{\frac{\log\vert\Theta\vert+\log\frac{2}{\delta}}{N}}\vee C\tilde{\sigma}^2\cdot\frac{\log\vert\Theta\vert+\log\frac{2}{\delta}}{N}$, $\epsilon_{N_p} = \sqrt{C\tilde{\sigma}^2}\cdot\sqrt{\frac{\log\vert\Theta\vert+\log\frac{2}{\delta}}{N_p}}\vee C\tilde{\sigma}^2\cdot\frac{\log\vert\Theta\vert+\log\frac{2}{\delta}}{N_p}$, $\tilde{\sigma}^2=\sigma^2+1$ and $C=8\sqrt{2}$.

% \yub{Give $\sigma^2+1$ a new name, say $\widetilde{\sigma}^2$, since they always appear together.}
% \yub{Give an absolute constant $C$ for all absolute constants in $\epsilon_N$ and $\epsilon_{N_p}$, then say $C\le 8\sqrt{2}$ or something like that.}

\label{thm:bound}
\end{theorem}

% \yub{Other than 2, the thm statement looks correct to me. After changing the correctness, you can spend some time to polish the notation (say, $\epsilon_{N_p}$ looks a bit weird).}

% \yub{we are not doing the continuous $\Theta$ version? (with covering techniques)? sounds fine to me though.}

Theorem~\ref{thm:bound} can be proved starting from the Bernstein’s inequality and the union bound inequality on squared zero-mean sub-Gaussian random variable. More detailed proof is in Appendix. As a large amount of unlabeled data (i.e., $N$) is usually easy to obtain by either manual collection or simulation, $\epsilon_N$ is not large after taking logarithm on the number of model parameters (i.e., $\vert\Theta\vert$), even though it is usually much larger than $N$. Additionally, $\log\vert\Theta\vert$ is not significantly larger than $N_p$ so that $\epsilon_{N_p}$ should not be large as well. For instance, in our experiments, around 45,000 samples and 14 million model parameters for motion generation result in rather small $\epsilon_N$ and $\epsilon_{N_p}$. In practice, we use $\xi=\rho\cdot\inf_{\theta\in\Theta}\gL_1(\theta)$ in \eqref{eq:obj_l2} (i.e., $\hat{\xi}=\rho\cdot\inf_{\theta\in\hat{\Theta}}\hat{\gL}_1(\theta)$ in Eq. \ref{eq:emp_l2}) to relax the constraint, where $\rho$ is an hyperparameter to keep the constraint within the confidence interval.

\section{Experiments}
\label{sec:exp}
% In this section, we firstly introduce datasets and their text annotation in Section~\ref{sec:data}. We then introduce the baseline comparisons in Section~\ref{sec:baseline}, followed by evaluation metrics in Section~\ref{sec:eval}. Then we present and discuss the evaluating results of generation quality in Section~\ref{sec:gen}, followed by the results for evaluating controllability in Section~\ref{sec:cont}. 
% \ning{Remove this paragraph if running out of space.}

\subsection{Datasets}
\label{sec:data}
We employ datasets from three modalities that may suffer from low-resource scenario. Details are in Appendix.

\textbf{Molecules}. We extract 130,831 molecules from QM9 dataset~\citep{ramakrishnan2014quantum} with six molecular properties: polarizability ($\alpha$), highest occupied molecular orbital energy ($\epsilon_{\text{HOMO}}$), lowest unoccupied molecular orbital energy ($\epsilon_{\text{LUMO}}$), the energy difference between HOMO and LUMO ($\Delta_{\epsilon}$), dipole moment ($\mu$) and heat capacity at 298.15K ($C_v$). 

\textbf{Motions}. We employ HumanML3D that contains textually re-annotating motions captured from AMASS~\citep{mahmood2019amass} and HumanAct12~\citep{guo2020action2motion}. It contains 14,616 motions annotated by 44,970 textual descriptions.

\textbf{Time Series}. We assemble 24 stocks from Yahoo Finance from their IPO date to July 8, 2023. We tailor data to the length of 120 by slicing on the opening price for every 120 days, and scale by min-max normalization following~\citet{yoon2019time}. Totally, 210,964 time series are produced. We extract their features including frequency, skewness, mean, variance, linearity (i.e., $R^2$), and the number of peaks via ``tsfresh" in Python. 

\subsection{Baseline models}
\label{sec:baseline}
We compare Text2Data with a representative classifier-free diffusion model in each modality as the baselines. 

\textbf{E(3) Equivariant Diffusion Model (EDM)}~\citep{hoogeboom2022equivariant}.To handle Molecule dataset, we employ EDM as the baseline. EDM utilizes an equivariant network to denoise diffusion processes by concurrently processing both continuous (i.e., atom coordinates) and categorical data (i.e., atom types). The controllability on molecular properties is realized by the classifier-free diffusion guidance conditioned on the embedding of text descriptions, which is encoded by a pretrained T5~\citep{2020t5} encoder (i.e., t5-large).

\textbf{Motion Diffusion Model (MDM)}~\citep{tevet2023human}. We use MDM, a classifier-free diffusion model, for text-to-human motion generation. The text descriptions are embedded by a pretrained T5 encoder to guide the motion generation, providing the mechanism for controllability. 

\textbf{Generation diffusion for time series (DiffTS)}. To generate time series, we design the classifier-free diffusion model conditioned on text embeddings encoded from pretrained T5 encoder. We employ the backbone of \cite{ho2022classifier} by substituting image data to one-dimensional time series, and replacing the U-Net with one-dimensional convolutional neural network.

For implementation, we train baselines on specific proportions of labeled data. Text2Data is modified from a pretraining+finetuning strategy following Eq. \ref{eq:emp_l2}. Additionally, we conduct ablation study where each baseline is still finetuned but without the constraint in Eq. \ref{eq:emp_l2} as a transfer learning-based baseline, or directly applied to augmented text-data pairs. Besides, we adapt from \cite{you2024diffusion} to form a semi-supervised framework, by training a classifier to predict molecular properties and generate pseudo labels for unlabelled molecules. More details are in Appendix.

\begin{figure*}[hbt!]
\centering
\begin{subfigure}[c]{.32\linewidth}
\includegraphics[width=\linewidth]{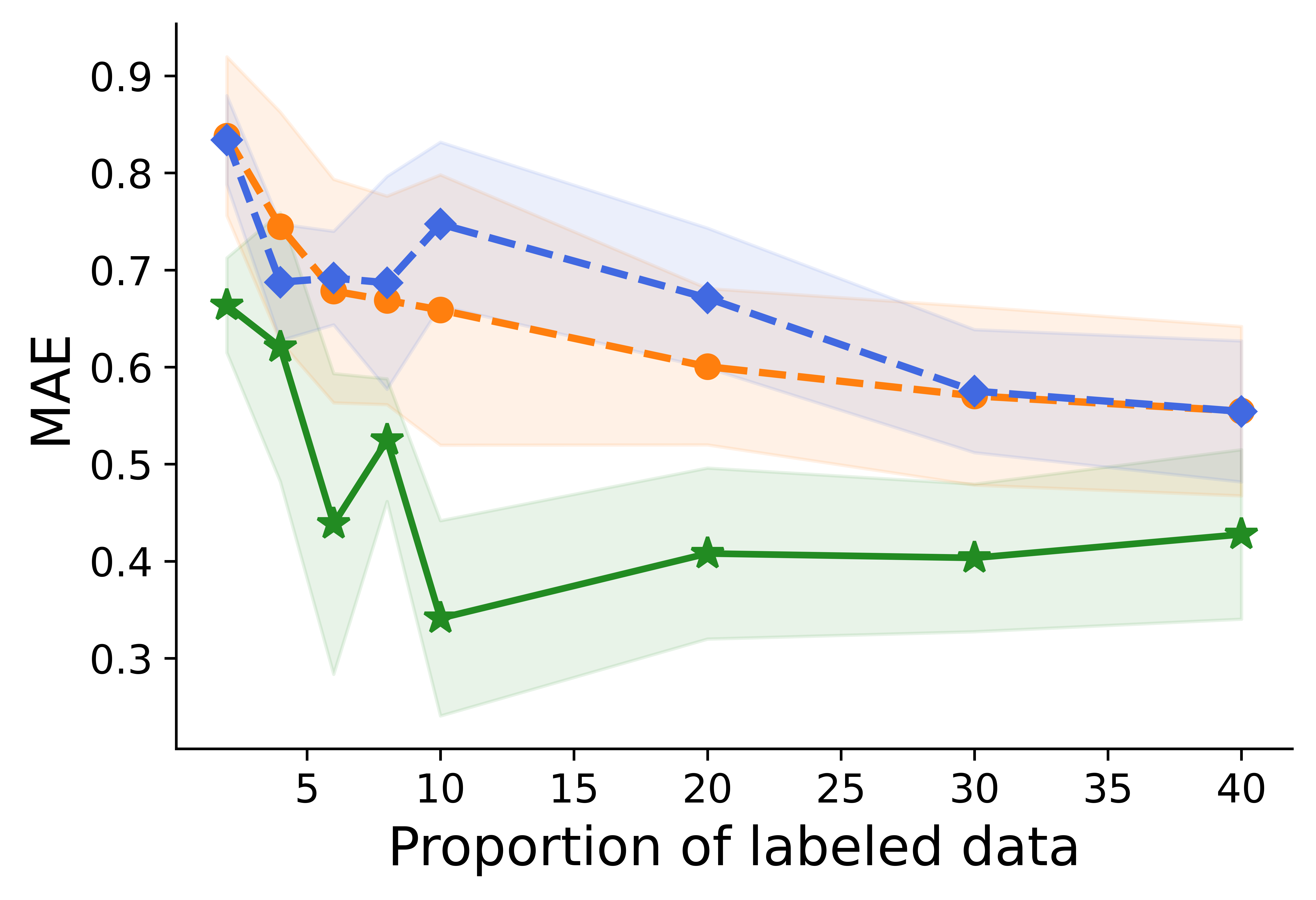}  
  \caption{$\alpha$ of generated molecules}
\end{subfigure}
\begin{subfigure}[c]{.32\linewidth}
  \includegraphics[width=\linewidth]{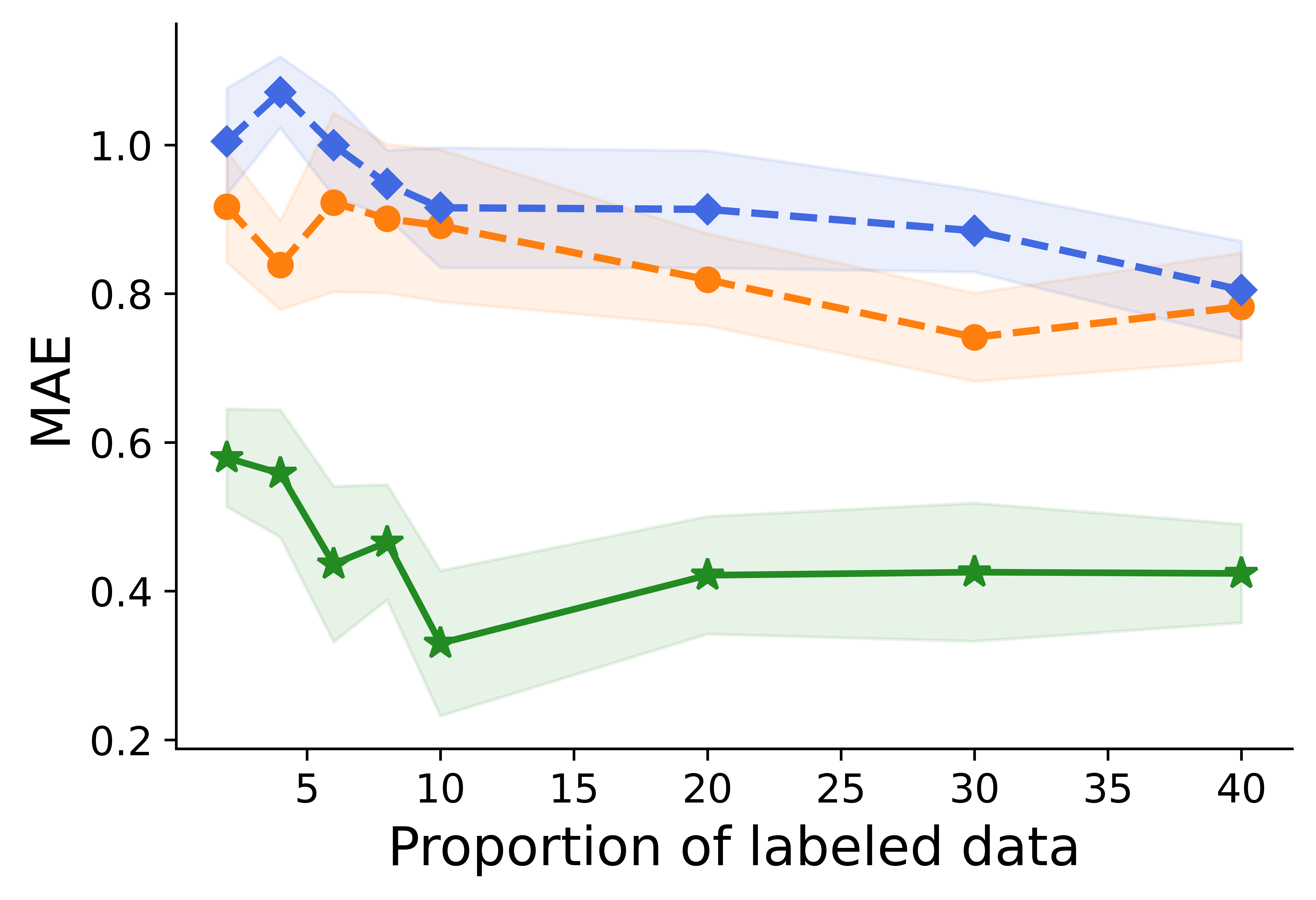}  
  \caption{$\epsilon_{HOMO}$ of generated molecules}
\end{subfigure}
\begin{subfigure}[c]{.32\linewidth}
  \includegraphics[width=\linewidth]{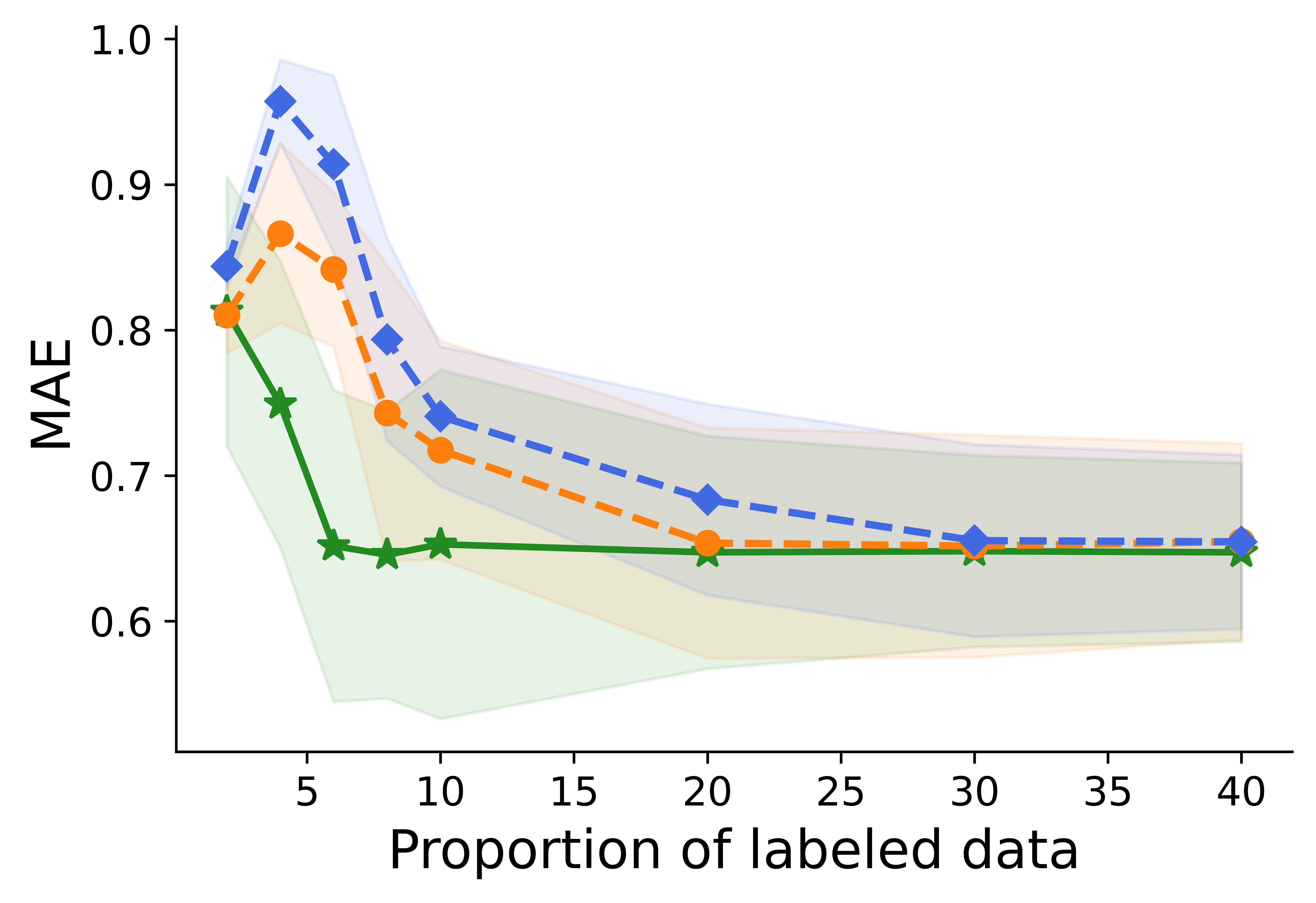}  
  \caption{$\epsilon_{LUMO}$ of generated molecules}
\end{subfigure}
\newline
\begin{subfigure}[c]{.32\linewidth}
  \includegraphics[width=\linewidth]{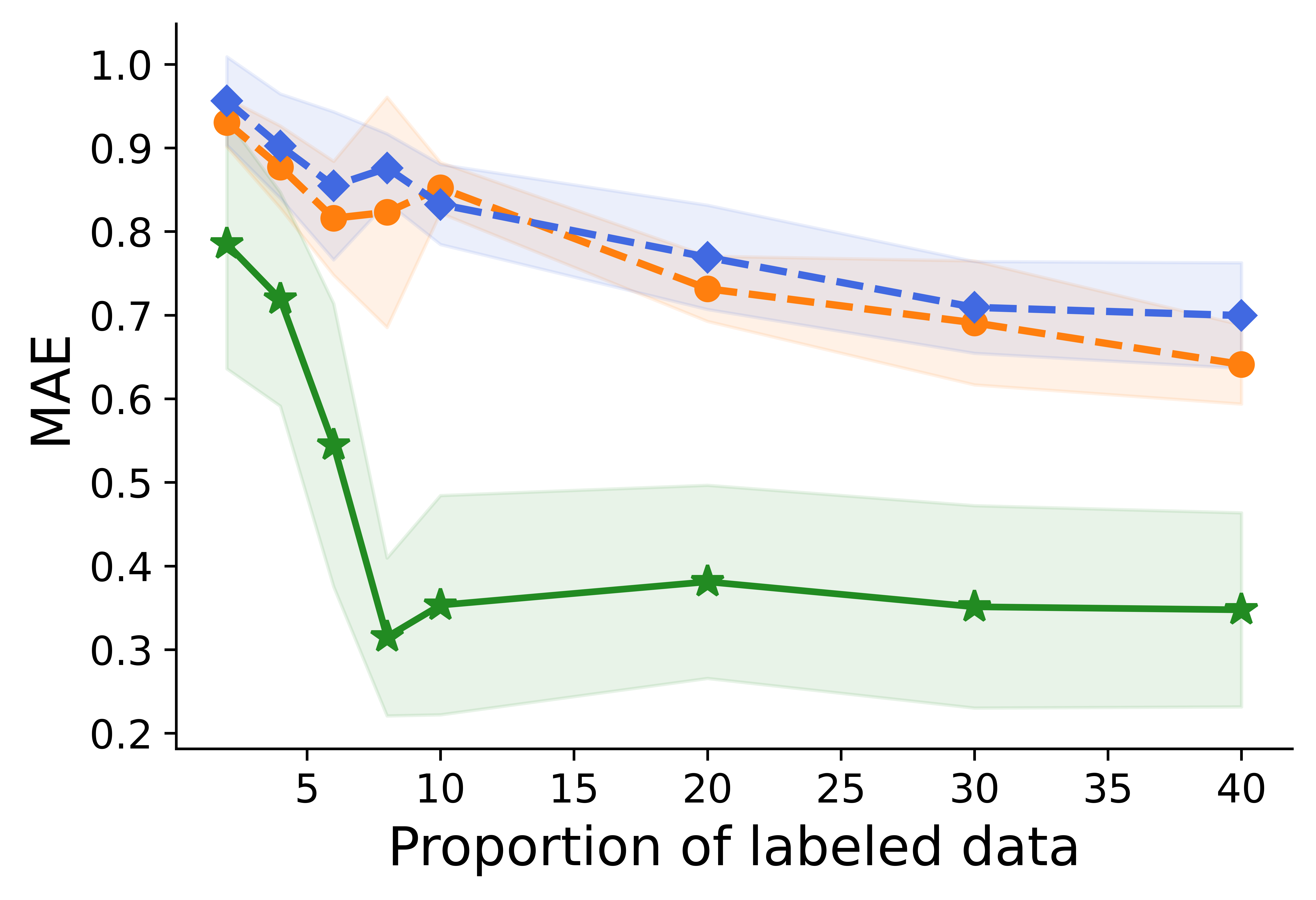}  
  \caption{$\Delta_{\epsilon}$ of generated molecules}
\end{subfigure}
\begin{subfigure}[c]{.32\linewidth}
  \includegraphics[width=\linewidth]{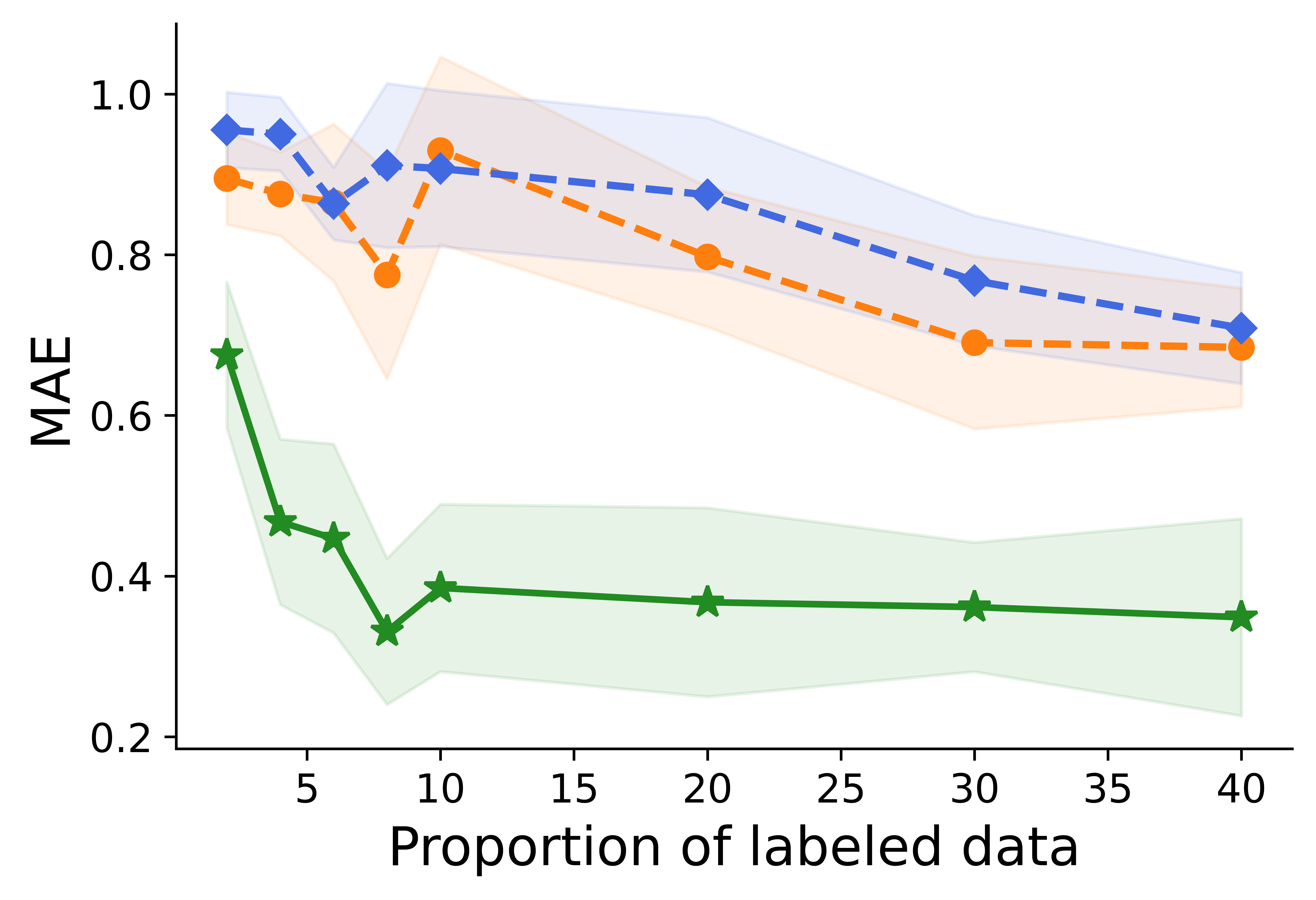}  
  \caption{$\mu$ of generated molecules}
\end{subfigure}
\begin{subfigure}[c]{.32\linewidth}
\includegraphics[width=\linewidth]{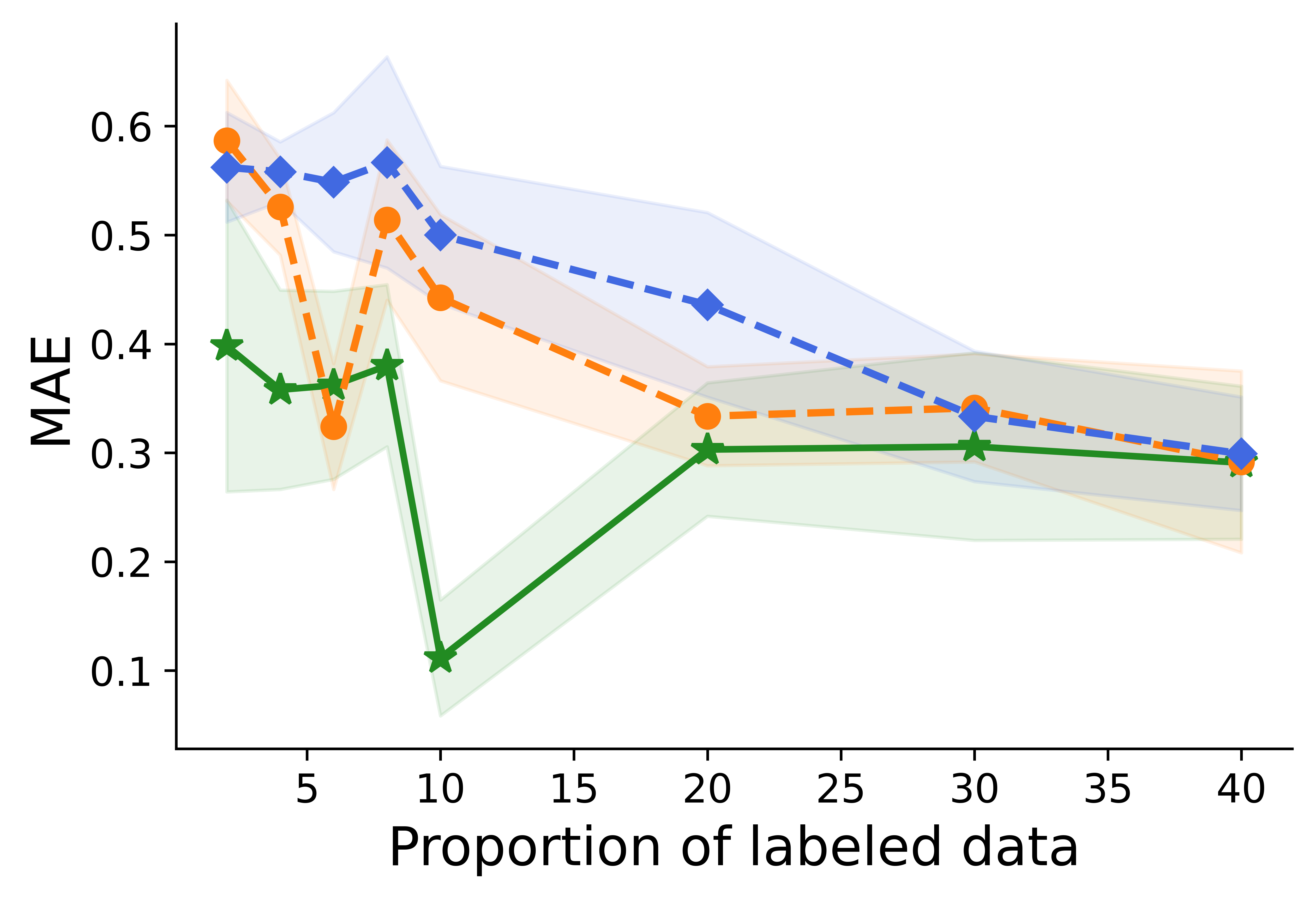}  
  \caption{$C_v$ of generated molecules}
\end{subfigure}
\caption{Evaluate controllability on Molecule dataset according to different proportions of paired data. Green solid line corresponds to Text2Data and two dashed lines are baseline comparisons, in which blue line is EDM and orange line is EDM-finetune. Properties of generated molecules are predicted by classifier $\phi_c$. MAE is computed between properties of generated molecules and intended properties. Lower MAE indicates better performance.}
\label{fig:mol_cond}
% \vspace{-5mm}
\end{figure*}

\begin{figure*}[hbt!]
\begin{center}
\includegraphics[width=0.85\textwidth]{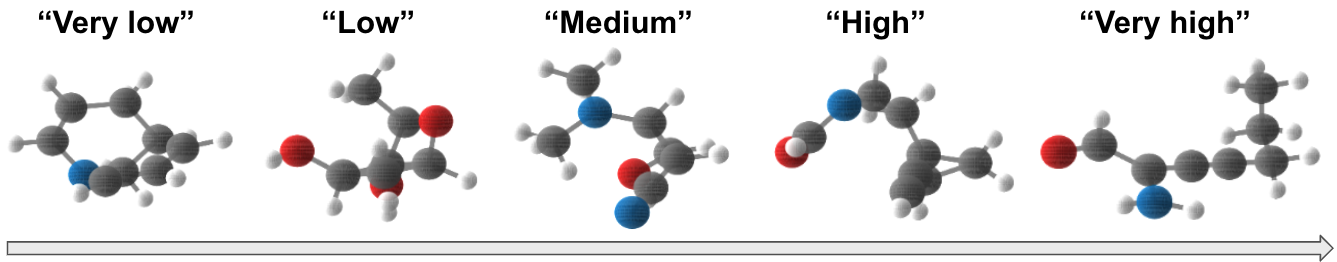}
\caption{Visualization of generated molecules when the polarizability increases from ``very low" to ``very high".}
\label{fig:mol_polar}
\end{center}
% \vspace{-3mm}
\end{figure*}

\subsection{Evaluation metrics}
\label{sec:eval}
We follow different strategies to evaluate both the controllability and the generation quality of the propose approach compared to other baselines.

\textbf{Controllability}. We compare the controllability between Text2Data and the baselines by the similarity between generated data and the ground truth. To assess the generated molecules, we follow \cite{hoogeboom2022equivariant} and train a classifier for each property to extract specific properties from generated data. Then, we calculate the \textit{Mean Absolute Error} (MAE) between the extracted property and the ground truth. To assess the controllability of motion generation, we compute \textit{R precision} and \textit{Multimodal distance} that measure the relevancy of the generated motions to the input prompts~\citep{guo2022generating}. To evaluate the controllability of time series generation, we extract properties via ”tsfresh” and compute the MAE between properties of generated data and that of ground truth. Additionally, we also visualize generated data according to the specified properties.

\textbf{Generation quality}. Generation quality varies based on the modality. For molecular generation, we compute \textit{$-\log$ likelihood} ($-\log p$)~\citep{hoogeboom2022equivariant} and \textit{validity} of generated molecules~\citep{jin2018junction}, \textit{molecular stability} and \textit{atom stability}~\citep{garcia2021n} to evaluate their overall generation quality. For motion generation, we use \textit{FID score} and \textit{Diversity} following~\citep{tevet2023human}. For time series generation, we follow \cite{yoon2019time} and \cite{lim2023regular} by drawing t-SNE plots to visualize the overlap between generated data and the ground-truth. Better model tends to have a larger overlap, indicating more similar distribution. 
% \ning{So no quantitative metric for time series?}
\begin{table*}[hbt!]
 \caption{Evaluate controllability on HumanML3D dataset by R Precision and Multimodal Distance according to different proportions of paired data.}
\centering
\begin{adjustbox}{width=0.95\textwidth,center}
\tiny
\begin{tabular}{c c ccc c ccc} 
\hline
\multirow{2}{*}{Proportion (\%)} && \multicolumn{3}{c}{R Precision $\uparrow$} &&
\multicolumn{3}{c}{Multimodal Dist. $\downarrow$} \\\cline{3-5} \cline{7-9} 
&& Text2Data & MDM-finetune & MDM && Text2Data & MDM-finetune & MDM  \\ \hline
2 &&0.34$\pm$0.01 &\textbf{0.37$\pm$0.01} &0.31$\pm$0.01 &&6.48$\pm$0.06 &\textbf{6.19$\pm$0.05} &6.67$\pm$0.02 \\
4 && 0.39$\pm$0.01&\textbf{0.42$\pm$0.01} &0.38$\pm$0.01 &&5.99$\pm$0.05 &\textbf{5.83$\pm$0.04} &6.01$\pm$0.04 \\
6 && 0.43$\pm$0.01& \textbf{0.43$\pm$0.01}&0.40$\pm$0.02 && 5.85$\pm$0.06&\textbf{5.78$\pm$0.05} &6.01$\pm$0.06 \\
8 && \textbf{0.44$\pm$0.01}& 0.43$\pm$0.01&0.42$\pm$0.01 && \textbf{5.65$\pm$0.04}&5.75$\pm$0.05 & 5.90$\pm$0.04\\
10 &&\textbf{0.45$\pm$0.01} &0.45$\pm$0.01 &0.44$\pm$0.01 &&\textbf{5.74$\pm$0.07} &5.76$\pm$0.07 &5.84$\pm$0.06\\
20 && \textbf{0.47$\pm$0.01}& 0.47$\pm$0.01& 0.45$\pm$0.01&& \textbf{5.61$\pm$0.04}& 5.68$\pm$0.11&5.73$\pm$0.14 \\
30 &&\textbf{0.48$\pm$0.01} &0.47$\pm$0.01 & 0.45$\pm$0.11 &&\textbf{5.61$\pm$0.05} &5.66$\pm$0.06 &5.80$\pm$0.09\\
40 && \textbf{0.49$\pm$0.01}& 0.46$\pm$0.01& 0.45$\pm$0.01&& \textbf{5.61$\pm$0.05}& 5.63$\pm$0.09& 5.90$\pm$0.04\\
\hline
\end{tabular}
\end{adjustbox}
\label{tab:mt_cond}
\vspace{-3mm}
\end{table*}
\normalsize

\begin{table*}[hbt!]
 \caption{Evaluate controllability on time series by MAE on testing set, according to different proportions of paired data. Lower MAE indicates better performance.}
\centering
\HUGE
\begin{adjustbox}{width=0.95\textwidth,center}
\begin{tabular}{c c ccc c ccc c ccc} 
\hline
\multirow{2}{*}{Proportion (\%)} && \multicolumn{3}{c}{Frequency ($\times 10^{-1}$)} && \multicolumn{3}{c}{Skewness} &&
\multicolumn{3}{c}{Mean ($\times 10^{-2}$)} \\\cline{3-5} \cline{7-9} \cline{11-13} 
&& Text2Data & DiffTS-finetune & DiffTS && Text2Data & DiffTS-finetune & DiffTS && Text2Data & DiffTS-finetune & DiffTS\\ \hline
2 &&\textbf{2.59$\pm$0.20} &2.62$\pm$0.20 &2.60$\pm$0.17 &&\textbf{1.68$\pm$0.20} &2.34$\pm$0.30 &1.84$\pm$0.22 && \textbf{0.63$\pm$0.39}&0.63$\pm$0.39 &0.79$\pm$0.41 \\
4 &&\textbf{2.55$\pm$0.18} &2.59$\pm$0.19 &2.59$\pm$0.19 &&\textbf{1.63$\pm$0.14} &1.77$\pm$0.29 & 2.80$\pm$0.28 &&\textbf{0.60$\pm$0.40} &0.61$\pm$0.39 & 0.73$\pm$0.40\\ 
6 &&\textbf{2.52$\pm$0.18} &2.57$\pm$0.19 & 2.57$\pm$0.19&&\textbf{1.00$\pm$0.16} &1.85$\pm$0.22 &1.81$\pm$0.21 &&\textbf{0.56$\pm$0.38} &0.58$\pm$0.38 & 0.71$\pm$0.36\\ 
8 &&\textbf{2.54$\pm$0.18} &2.56$\pm$0.19 & 2.57$\pm$0.18&&\textbf{1.10$\pm$0.18} &1.56$\pm$0.18 &1.78$\pm$0.09 &&\textbf{0.57$\pm$0.38} &0.63$\pm$0.40 & 0.62$\pm$0.36\\
10 &&\textbf{2.54$\pm$0.19} &2.57$\pm$0.18 &2.55$\pm$0.20 &&\textbf{0.87$\pm$0.13} &1.20$\pm$0.17 &1.12$\pm$0.11 &&\textbf{0.55$\pm$0.37} &0.57$\pm$0.36 & 0.62$\pm$0.39\\
20 &&\textbf{2.54$\pm$0.18} &2.55$\pm$0.18 &2.58$\pm$0.21 &&\textbf{1.05$\pm$0.15} &1.06$\pm$0.12 &1.26$\pm$0.14 &&\textbf{0.55$\pm$0.40} &0.57$\pm$0.37 & 0.65$\pm$0.38\\
30 &&\textbf{2.53$\pm$0.18} &2.56$\pm$0.17 &2.56$\pm$0.18 &&\textbf{1.03$\pm$0.12} &1.16$\pm$0.25 &1.71$\pm$0.23 &&\textbf{0.51$\pm$0.33} &0.53$\pm$0.33 & 0.59$\pm$0.34\\ 
40 &&\textbf{2.53$\pm$0.18} &2.55$\pm$0.18 &2.55$\pm$0.19 &&\textbf{1.03$\pm$0.11} &1.15$\pm$0.24 &1.19$\pm$0.18 &&\textbf{0.51$\pm$0.33} &0.57$\pm$0.32 & 0.57$\pm$0.36\\ 
\hline
\end{tabular}
\end{adjustbox}
\label{tab:ts_cond}
\end{table*}
\normalsize
\subsection{Comparisons on controllability}
\label{sec:cont}
Figure~\ref{fig:mol_cond} illustrates the MAE trend between properties of generated molecules and the intended one as the proportion of labeled training data rises. Text2Data achieves superior performance than EDM-finetune and EDM on all properties by a remarkable margin. The results also indicate that certain properties, such as $\epsilon_{LUMO}$ and $Cv$, are more readily controllable. For these properties, the performance of the three models converges as the amount of labeled training data becomes sufficiently large (Figure~\ref{fig:mol_cond}). We further increase the proportion of available labels in the dataset up to $100\%$ and, as indicated in Appendix Table~\ref{tab:mol_prop}, the MAE keeps increasing when more labels are involved during training, and it gradually converges in the end. Appendix Table~\ref{tab:mol_ablation} also suggests that Text2Data surpasses the data augmentation-based method, which may suffer from the potentially poor alignment between text and data and the potential overfitting.

We depict the molecules generated as the text descriptor for polarizability shifts from ``very low" to ``very high" in Figure~\ref{fig:mol_polar}. Polarizability indicates the inclination of the molecule to form an electric dipole moment under an external electric field. As $\alpha$ values rise, we expect to see molecules with less symmetrical forms, as evidenced in Figure~\ref{fig:mol_polar}. This trend suggests the validity of generated molecules by Text2Data and its fine-grained controllability.

As suggested in Table~\ref{tab:mt_cond}, Text2Data also outperforms MDM-finetune and MDM in the controllable generation of motions from text descriptions. While MDM-finetune is slightly better than Text2Data when the proportion of labeled training data is small-owing to milder catastrophic forgetting during finetuning with a smaller sample size-Text2Data consistently surpasses both MDM-finetune and MDM as the volume of labeled training data increases. Specifically, in this situation, Text2Data surpasses MDM-finetune and MDM in R Precision with average margins of 2.31$\%$ and 5.57$\%$, respectively, and in Multimodal Distance with average margins of 0.93$\%$ and 3.30$\%$, respectively. The results also indicate that an increase in labeled training data enhances the performance of controllability, which is expected as more supervision is involved. 

Additionally, we evaluate controllability of Text2Data, along with its baseline comparisons, utilizing MAE to measure the congruence between the property of generated data and the intended one within the Time Series dataset. As indicated in Table~\ref{tab:ts_cond}, Text2Data consistently excels over two baselines, DiffTS-finetune and DiffTS, across all three properties assessed. Results of another three properties are presented in Appendix Table~\ref{tab:ts_cond2}, which suggests the similar conclusion. Specifically, Text2Data and DiffTS-finetune both show a marked improvement over DiffTS in controlling frequency, variance, and skewness. They also exhibit a slight edge in controlling mean, number of peaks, and linearity. The enhanced performance of Text2Data correlates with its proficiency in alleviating the issue of catastrophic forgetting while maintaining a pursuit of controllability.

\subsection{Comparisons on generation quality}
\label{sec:gen}
Text2Data not only demonstrates superior performance in controllable text-to-data generation but also sustains competitive generation quality relative to baseline models.

When generating molecules from Text2Data and its baseline comparisons, as shown in Appendix Table~\ref{tab:mol_gen1}, we compute $-\log p$ and validity to evaluate generation quality. The performance of Text2Data is consistently better. It surpasses EDM-finetune and EDM by average margins of 19.07$\%$ and 58.03$\%$, respectively. It is 1.98$\%$ and 10.59$\%$ better than EDM-finetune and EDM on average, respectively, regarding validity of molecules. Besides, we evaluate Text2Data compared with EDM-finetune and EDM on molecular stability and atom stability (Appendix Table~\ref{tab:mol_gen2}). Text2Data exceeds EDM-finetune and EDM by average margins of 2.34$\%$ and 17.31$\%$, respectively, regarding molecular stability. It is also 0.29$\%$ and 1.21$\%$ better than EDM-finetune and EDM on average, respectively, regarding atom stability. The consistent improvements on all the three models result from our superior performance of Text2Data on properties (e.g., molecular stability) that are hard to control.

As indicated in Appendix Table~\ref{tab:mt_gen}, quantitative assessment of motion generation from text shows that Text2Data surpasses baseline methods in both quality and diversity. Particularly, Text2Data outperforms MDM-finetune and MDM by 2.73$\%$ and 36.05$\%$ on average, respectively, regarding FID. For diversity, Text2Data surpasses MDM-finetune and MDM by average margins of 0.81$\%$ and 3.71$\%$, respectively. Enhanced performance is derived from Text2Data to fully leverage all samples in the dataset, while effectively mitigating catastrophic forgetting during finetuning.

We evaluate the overall quality of generating time series by making t-SNE plots of generated time series against ground truth. Substantial overlap between the generated time series and the ground truth suggests a closer distribution alignment, implying a better performance of Text2Data. As demonstrated in Appendix Figure~\ref{fig:tsne}, specifically, the red pattern represents the t-SNE of the ground-truth time series, whereas the blue pattern represents the t-SNE of the generated time series according to the same text description. Compared with DiffTS-finetune and DiffTS, Text2Data corresponds to the largest overlap between distributions of the generated and the ground-truth time series, suggesting its superior ability to precisely generate data according to the text description. The non-overlapping part may result from the diversity of generated time series or some other properties not labeled in the dataset so that not controlled by text description. The inferior performance of DiffTS stems from its training solely on labeled data, potentially leading to an incomplete understanding of the overall data distribution and a risk of overfitting due to limited size of labeled data. DiffTS may only partially capture the data distribution based on textual descriptions due to its susceptibility to catastrophic forgetting, which also heightens the risk of overfitting. 

\section{Conclusion}
\label{sec:conclusion}
We propose Text2Data to improve quality and property control of text-to-data generation for various modalities in low-resource scenarios using diffusion models. Text2Data uses unlabeled data to capture the prevailing data distribution via unsupervised diffusion model. It is then finetuned on text-labeled data, with a novel constraint optimization-based learning objective to ensure controllability while reducing catastrophic forgetting. Experiments show consistently superior performance to recent baselines. While Text2Data is presented as a diffusion-based framework in this article, it can be seamlessly adapted to other generative models. 
\newpage
\appendix
\onecolumn
\section{Derivation of diffusion model}
\label{sec:diff}
We introduce the framework of the classifier-free diffusion model~\cite{ho2022classifier} that we use, including the forward process, reverse process and its learning objective. Let $\rvx^{0:T}$ be $\{\rvx, \rvx^{(1)}, ..., \rvx^{(T)}\}$. Diffusion models have the form $p_{\theta}(\rvx)=\int p_{\theta}(\rvx, \rvx^{1:T})d\rvx^{0:T}$, where $\rvx^{(1)}, \rvx^{(2)}, ..., \rvx^{(T)}$ are latents of the same dimensionality as the data $\rvx$.
\subsection{Forward process}
First, the data $\rvx$ is diffused up to time $T$ by adding noise. The posterior $q(\rvx^{1:T}\vert \rvx)$ is fixed to a Markov Chain that gradually adds noise to the data according to a variance schedule $\beta_1, \beta_2, ..., \beta_T$:
\begin{align}
    q(\rvx^{(1)}, \rvx^{(2)}, ..., \rvx^{(T)}\vert \rvx)&=\prod_{t=1}^T q(\rvx^{(t)}\vert \rvx^{(t-1)}) \\
    q(\rvx^{(t)}\vert \rvx^{(t-1)})&= \mathcal N(\rvx^{(t)}; \sqrt{1-\beta_t}\rvx_{t-1}, \beta_t\bm I) 
    \label{eq:mk}
\end{align}
Let $\alpha_t=1-\beta_t$ and $\Bar{\alpha}_t=\prod_{s=1}^t\alpha_s$, we can derive:
\begin{align}
    q(\rvx^{(t)}\vert\rvx)=\mathcal N(\rvx^{(t)};\sqrt{\bar{\alpha}_t}\rvx, (1-\Bar{\alpha}_t)\bm I)
    \label{eq:diff}
\end{align}

\subsection{Reverse process}
The diffusion model aims to learn $p_{\theta}(\rvx^{0:T}, \rvc)$ for the paired data s.t. $\rvx, \rvc\in\gD_p$, which is defined as a Markov Chain with learned Gaussian transitions starting from $p(\rvx^{(T)})=\mathcal N(\rvx^{(T)}; \bm 0, \bm I)$.
\begin{align}
    p_{\theta}(\rvx^{0:T}, \rvc)&=p(\rvx^{(T)}, \rvc)\prod_{t=1}^{T}p_{\theta}(\rvx^{(t-1)}\vert \rvx^{(t)}, \rvc)=p(\rvx^{(T)}, \rvc)\prod_{t=1}^{T}p_{\theta}(\rvx^{(t-1)}\vert \rvx^{(t)}, \rvc)\\
    p_{\theta}(\rvx^{(t-1)}\vert\rvx^{(t)}, \rvc)&=\gN(\rvx^{(t-1)};\meps_{\theta}(\rvx^{(t)}, \rvc, t), \Sigma_{\theta}(\rvx^{(t)}, \rvc, t))
\end{align}

\subsection{Derivation of learning objective}
\label{sec:diff_obj}
To maximize the joint likelihood $p_{\theta}(\rvx,\rvc)$, we derive evidence variational lower bound (ELBO) as the learning objective. Specifically, we have:
\begin{align}
    \log p_{\theta}(\rvx, \rvc) &= \log \int p_{\theta}(\rvx, \rvx^{(1)}, ..., \rvx^{(T)}, \rvc) d\rvx^{1:T} \\
    &=\log \int p_{\theta}(\rvx^{0:T}, \rvc)d\rvx^{1:T} \\
    &=\log \int \frac{p_{\theta}(\rvx^{0:T}, \rvc) q(\rvx^{1:T}\vert \rvx)}{q(\rvx^{1:T}\vert \rvx)}d\rvx^{1:T} \\ 
    &=\log\mathbb E_{q(\rvx^{1:T}\vert \rvx)}[\frac{p_{\theta}(\rvx^{0:T}, \rvc)}{q(\rvx^{1:T}\vert \rvx)}] \\
    \label{eq:logp}
\end{align}

Based on Jensen’s inequality, the ELBO of Eq.~\ref{eq:logp} is as below: 
\begin{align}
    \log p_{\theta}(\rvx, \rvc) &= \log\mathbb E_{q(\rvx^{1:T}\vert \rvx)}[\frac{p_{\theta}(\rvx^{0:T}, \rvc)}{q(\rvx^{1:T}\vert \rvx)}]\\
    &\ge\mathbb E_{q(\rvx^{1:T}\vert \rvx)}[\log \frac{p_{\theta}(\rvx^{0:T}, \rvc)}{q(\rvx^{1:T}\vert \rvx)}]\\
    &=\mathbb E_{q(\rvx^{1:T}\vert \rvx)}[\log p_{\theta}(\rvx^{(T)}, \rvc)+\sum_{t\ge 1}\log \frac{p_{\theta}(\rvx^{(t-1)}\vert\rvx^{(t)}, \rvc)}{q(\rvx^{(t)}\vert\rvx^{(t-1)})}]
    \label{eq:vlbo}
\end{align}

Since $p(\rvx^{(T)})=\mathcal N(\rvx^{(T)}; \bm 0, \bm I)$ and not dependent on $\rvc$, as a result, $\log p(\rvx^{(T)}, \rvc)=\log p(\rvx^{(T)}\vert\rvc)p(\rvc) =\log p(\rvx^{(T)})p(\rvc) = \log p(\rvx^{(T)}) + \log p(\rvc)$. Then Eq.~\ref{eq:vlbo} above becomes:
\begin{align}
    \log p_{\theta}(\rvx, \rvc)&\ge\mathbb E_{q(\rvx^{1:T}\vert \rvx)}[\log p_{\theta}(\rvx^{(T)}, \rvc)+\sum_{t\ge 1}\log \frac{p_{\theta}(\rvx^{(t-1)}\vert\rvx^{(t)}, \rvc)}{q(\rvx^{(t)}\vert\rvx^{(t-1)})}]\\
    &=\mathbb E_{q(\rvx^{1:T}\vert \rvx)}[\log p(\rvx^{(T)})+\log p(\rvc)+\sum_{t\ge 1}\log \frac{p_{\theta}(\rvx^{(t-1)}\vert\rvx^{(t)}, \rvc)}{q(\rvx^{(t)}\vert\rvx^{(t-1)})}]
\end{align}

% \hw{Not quite sure if we can get to (8) from (7). Do we want to multiply $p(x_{1:T}|x_0,c)$ instead of $p(x_{1:T}|x_0)$ ?}\\

% \hw{can we provide more details from (8) to (9)? }

Given the posterior formed as the Markov Chain (Eq.~\ref{eq:mk}), we have $q(\rvx_t\vert \rvx_{t-1}) = q(\rvx^{(t)}\vert \rvx^{(t-1)}, \rvx)$. Therefore, we minimize the loss for the paired data of $\{\rvx, \rvc\}\in\gD_p$:
\begin{align}
    \gL_{ELBO}&=-\mathbb E_{q(\rvx^{1:T}\vert \rvx)}[\log p(\rvx^{(T)})+\log p(\rvc)+\sum_{t\ge 1}\log \frac{p_{\theta}(\rvx^{(t-1)}\vert\rvx^{(t)}, \rvc)}{q(\rvx^{(t)}\vert\rvx^{(t-1)})}] \\
    &=\mathbb E_{q(\rvx^{1:T}\vert \rvx)}[-\log p(\rvx^{(T)}) -\log p(\rvc)-\sum_{t> 1}\log \frac{p_{\theta}(\rvx^{(t-1)}\vert\rvx^{(t)}, \rvc)}{q(\rvx^{(t)}\vert\rvx^{(t-1)})} -\log \frac{p_{\theta}(\rvx\vert\rvx^{(1)}, \rvc)}{q(\rvx^{(1)}\vert\rvx)}]\\
    &=\mathbb E_{q(\rvx^{1:T}\vert \rvx)}[-\log p(\rvx^{(T)}) -\log p(\rvc)-\sum_{t> 1}\log \frac{p_{\theta}(\rvx^{(t-1)}\vert\rvx^{(t)}, \rvc)}{q(\rvx^{(t-1)}\vert\rvx^{(t)}, \rvx)}\cdot\frac{q(\rvx^{(t-1)}\vert \rvx)}{q(\rvx^{(t)}\vert \rvx)} -\log \frac{p_{\theta}(\rvx\vert\rvx^{(1)}, \rvc)}{q(\rvx^{(1)}\vert\rvx)}]\\
    &=\mathbb E_{q(\rvx^{1:T}\vert \rvx)}[-\log \frac{p(\rvx^{(T)})}{q(\rvx^{(T)}\vert \rvx)} -\log p(\rvc)-\sum_{t> 1}\log \frac{p_{\theta}(\rvx^{(t-1)}\vert\rvx^{(t)}, \rvc)}{q(\rvx^{(t-1)}\vert\rvx^{(t)}, \rvx)} -\log p_{\theta}(\rvx\vert\rvx^{(1)}, \rvc)] \\
    &=D_{KL}(q(\rvx^{(T)}\vert\rvx)\vert\vert p(\rvx^{(T)}))+\sum_{t>1} D_{KL}(q(\rvx^{(t-1)}\vert \rvx^{(t)}, \rvx)\vert\vert p_{\theta}(\rvx^{(t-1)}\vert \rvx^{(t)}, \rvc)) -\mathbb E_{q(\rvx^{1:T}\vert \rvx)}[\log p_{\theta}(\rvx\vert \rvx^{(1)}, \rvc)]-\log p(\rvc) \\
    &=L_T +\sum_{t>1}L_{t-1} +L_0 - \log p(\rvc) \\
    &\propto L_T +\sum_{t>1}L_{t-1} +L_0,
    \label{eq:l_paired}
\end{align}
where we have: 
\begin{align}
    L_T &= D_{KL}(q(\rvx^{(T)}\vert\rvx)\vert\vert p(\rvx^{(T)})) \\
    L_{t-1} &= D_{KL}(q(\rvx^{(t-1)}\vert \rvx^{(t)}, \rvx)\vert\vert p_{\theta}(\rvx^{(t-1)}\vert \rvx^{(t)}, \rvc))\\
    L_0 &= -\mathbb E_{q(\rvx^{1:T}\vert \rvx)}[\log p_{\theta}(\rvx\vert \rvx^{(1)}, \rvc)]
\end{align}

We can directly derive $q(\rvx^{(t-1)}\vert \rvx^{(t)}, \rvx)$ as:
\begin{align}
    q(\rvx^{(t-1)}\vert \rvx^{(t)}, \rvx) &= \frac{q(\rvx^{(t-1)}, \rvx^{(t)}, \rvx)}{q(\rvx^{(t)}, \rvx)} \\
    &=\frac{q(\rvx^{(t)}\vert \rvx^{(t-1)}, \rvx)q(\rvx^{(t-1)}, \rvx)}{q(\rvx^{(t)}, \rvx)} \\
    &=\frac{q(\rvx^{(t)}\vert\rvx^{(t-1)})q(\rvx^{(t-1)}\vert \rvx)}{q(\rvx^{(t)}\vert \rvx)} \\
    &\propto \exp[-\frac{1}{2}(\frac{(\rvx^{(t)} - \sqrt{\alpha_t}\rvx^{(t-1)2})}{\beta_t}+\frac{(\rvx^{(t-1)}-\sqrt{\Bar{\alpha}_{t-1}}\rvx)^2}{1-\Bar{\alpha}_{t-1}}-\frac{(\rvx^{(t)}-\sqrt{\Bar{\alpha}_t}\rvx)^2}{1-\Bar{\alpha}_t})]\\
    &=\exp[-\frac{1}{2}((\frac{\alpha_t}{\beta_t} + \frac{1}{1-\Bar{\alpha}_{t-1}})\rvx^{(t-1)2}-2(\frac{\sqrt{\alpha_t}}{\beta_t}\rvx^{(t)}+\frac{\sqrt{\Bar{\alpha}_{t-1}}}{1-\Bar{\alpha}_{t-1}}\rvx)\rvx^{(t-1)}+C(\rvx^{(t)}, \rvx))\\
    &=\exp[-\frac{1}{2}(A\rvx^{(t-1)2}-2B\rvx^{(t-1)}+C(\rvx^{(t)}, \rvx))],
    \label{eq:mu}
\end{align}
where $A = \frac{\alpha_t}{\beta_t} + \frac{1}{1-\Bar{\alpha}_{t-1}}$ and $B = \frac{\sqrt{\alpha_t}}{\beta_t}\rvx^{(t)}+\frac{\sqrt{\Bar{\alpha}_{t-1}}}{1-\Bar{\alpha}_{t-1}}\rvx$. Then we have:
\begin{align}
    \Sigma = \frac{1}{A} = \frac{1}{\frac{\alpha_t}{\beta_t} + \frac{1}{1-\Bar{\alpha}_{t-1}}} = \frac{(1-\Bar{\alpha}_{t-1})\beta_t}{(1-\Bar{\alpha}_{t-1})\alpha_t+\beta_t}=\frac{(1-\Bar{\alpha}_{t-1})\beta_t}{\alpha_t + \beta_t-\Bar{\alpha}_{t-1}\alpha_t}=\frac{(1-\Bar{\alpha}_{t-1})\beta_t}{1-\Bar{\alpha}_t}
\end{align}
Also we have:
\begin{align}
    \mu(\rvx^{(t)}, \rvx)&=\frac{B}{A} = B\Sigma\\
    &=(\frac{\sqrt{\alpha_t}}{\beta_t}\rvx^{(t)}+\frac{\sqrt{\Bar{\alpha}_{t-1}}}{1-\Bar{\alpha}_{t-1}}\rvx)\Sigma \\ 
    &=(\frac{\sqrt{\alpha_t}}{\beta_t}\rvx^{(t)}+\frac{\sqrt{\Bar{\alpha}_{t-1}}}{1-\Bar{\alpha}_{t-1}}\rvx)\frac{(1-\Bar{\alpha}_{t-1})\beta_t}{1-\Bar{\alpha}_t}\\
    &=\frac{(1-\Bar{\alpha}_{t-1})\sqrt{\alpha_{t}}}{1-\Bar{\alpha}_t}\rvx^{(t)} + \frac{\sqrt{\Bar{\alpha}_{t-1}}\beta_t}{1-\Bar{\alpha}_t}\rvx
    \label{eq:mu}
\end{align}
Based on Eq.~\ref{eq:diff}, we can reparameterize $\rvx^{(t)}$ as: $\rvx^{(t)}=\sqrt{\Bar{\alpha}_t}\rvx-\sqrt{1-\Bar{\alpha}_t}\meps$ where $\meps\sim \mathcal N(\bm 0, \bm I)$. Then, plug $\rvx = \frac{1}{\sqrt{\Bar{\alpha}_t}}(\rvx^{(t)} - \sqrt{1-\Bar{\alpha}_t}\meps)$ in to Eq.~\ref{eq:mu}, we have:
\begin{align}
\mu(\rvx^{(t)}, \rvx)&=\frac{(1-\Bar{\alpha}_{t-1})\sqrt{\alpha_{t}}}{1-\Bar{\alpha}_t}\rvx^{(t)} + \frac{\sqrt{\Bar{\alpha}_{t-1}}\beta_t}{1-\Bar{\alpha}_t}\frac{1}{\sqrt{\Bar{\alpha}_t}}(\rvx^{(t)} - \sqrt{1-\Bar{\alpha}_t}\meps) \\
&=\frac{(1-\Bar{\alpha}_{t-1})\sqrt{\alpha_{t}}}{1-\Bar{\alpha}_t}\rvx^{(t)} +\frac{\beta_t}{(1-\Bar{\alpha}_t)\sqrt{\alpha_t}}(\rvx^{(t)} - \sqrt{1-\Bar{\alpha}_t}\meps) \\
&=\frac{(1-\Bar{\alpha}_{t-1})\alpha_t+\beta_t}{(1-\Bar{\alpha}_t)\sqrt{\alpha_t}}\rvx^{(t)}-\frac{\beta_t}{\sqrt{1-\Bar{\alpha}_t}\sqrt{\alpha_t}}\meps\\
&=\frac{(\alpha_t+\beta_t)-\Bar{\alpha}_{t-1}\alpha_t}{(1-\Bar{\alpha}_t)\sqrt{\alpha_t}}\rvx^{(t)}-\frac{\beta_t}{\sqrt{1-\Bar{\alpha}_t}\sqrt{\alpha_t}}\meps\\
&=\frac{1-\Bar{\alpha}_t}{(1-\Bar{\alpha}_t)\sqrt{\alpha_t}}\rvx^{(t)}-\frac{\beta_t}{\sqrt{1-\Bar{\alpha}_t}\sqrt{\alpha_t}}\meps\\
&=\frac{1}{\sqrt{\alpha_t}}(\rvx^{(t)}-\frac{\beta_t}{\sqrt{1-\Bar{\alpha}_t}}\meps)
\end{align}

Therefore, $q(\rvx^{(t-1)}\vert \rvx^{(t)}, \rvx)\sim \mathcal N(\rvx^{(t-1)}; \frac{1}{\sqrt{\alpha_t}}(\rvx^{(t)}-\frac{\beta_t}{\sqrt{1-\Bar{\alpha}_t}}\meps), \frac{(1-\Bar{\alpha}_{t-1})\beta_t}{1-\Bar{\alpha}_t}\bm I)$, where $\meps\sim \gN(\bm 0, \bm I)$. We fix $\beta_t$ to constants. As a result, the posterior $q(\cdot)$ has no learnable parameters, and $L_T$ does not need to be trained and can be disregarded during training.

For $p_{\theta}(\rvx^{(t-1)}\vert\rvx^{(t)}, \rvc)=\mathcal N(\rvx^{(t-1)};\meps_{\theta}(\rvx^{(t)}, \rvc, t), \Sigma_{\theta}(\rvx^{(t)}, \rvc, t))$ where $1<t<T$, let $\Sigma_{\theta}(\rvx^{(t)}, \rvc, t)=\sigma_t^2\bm I$ be time-dependent constants. With $p_{\theta}(\rvx^{(t-1)}\vert\rvx^{(t)}, \rvc)=\mathcal N(\rvx^{(t-1)};\meps_{\theta}(\rvx^{(t)}, \rvc, t), \sigma_t^2\bm I)$, we have:

\begin{align}
    L_{t-1}=\mathbb E_{q(\rvx^{1:T}\vert \rvx)}[\frac{1}{2\sigma_t^2}\vert\vert\Tilde{\meps}_t(\rvx^{(t)}, \rvx) - \meps_{\theta}(\rvx^{(t)}, \rvc, t)\vert\vert^2] + C, \ \ \text{s.t.} \ \ \{\rvx, \rvc\}\in\mathcal D_{p}
\end{align}
where $C$ is a constant not related to $\rvc$, $\Tilde{\meps}_t(\rvx^{(t)}, \rvx)=\frac{1}{\sqrt{\alpha_t}}(\rvx^{(t)}-\frac{\beta_t}{\sqrt{1-\bar \alpha}_t}\meps)$ based on Eq.~\ref{eq:mu}, $\meps\sim\mathcal N(\bm 0, \bm I)$. Plug $\Tilde{\meps}_t(\rvx^{(t)}, \rvx)=\frac{1}{\sqrt{\alpha_t}}(\rvx^{(t)}-\frac{\beta_t}{\sqrt{1-\bar {\alpha}_t}}\meps)$ and $\Tilde{\meps}_t(\rvx^{(t)}, \rvx)=\frac{1}{\sqrt{\alpha_t}}(\rvx^{(t)}-\frac{\beta_t}{\sqrt{1-\bar {\alpha}_t}}\meps_\theta(\rvx^{(t)}, \rvc, t))$, we have:
\begin{align}
    L_{t-1}=\mathbb E_{q(\rvx^{1:T}\vert \rvx)}[\frac{\beta_t}{2\sigma_t^2\sqrt{\alpha_t(1-\bar {\alpha}_t)}}\vert\vert\meps - \meps_\theta(\rvx^{(t)}, \rvc, t)\vert\vert^2] + C, \ \ \text{s.t.} \ \ \{\rvx, \rvc\}\in\gD_{p}
\end{align}

A simpler version of empirical loss in our implementation:
\begin{align}
    \mathcal L_1(\theta, \rvx, \rvc)=\mathbb E_{\rvx,\rvc\sim p_{\gD_p}(\rvx,\rvc), t}\vert\vert\meps-\meps_{\theta}(\rvx^{(t)}, \rvc, t)\vert\vert^2,
\end{align}
which is the expression of $\gL_2(\theta)$ in Eq.~\ref{eq:obj_l2}. When compute $\gL_1(\theta)$ in Eq.~\ref{eq:obj_l2}, replace $\rvc$ with $\emptyset$. When it comes to classifier-free diffusion model, we can train unconditional denoising diffusion model via $\gL_1(\theta)$ together with the conditional model via $\gL_2(\theta)$. The occurrence of $\emptyset$ is sampled with the probability $p_{uncond}$.

\section{Lexicographic optimization}
\label{sec:lex}
The model is firstly trained on unlabeled data, and then finetuned on labeled data under the framework of lexicographic optimization~\citep{gong2021bi}. Lexicographic optimization is used to balance the optimization of the learning objective (i.e., $\hat{\gL}_2(\theta)$) and its constraint (i.e., $\hat{\gL}_1'(\theta)$). Specifically, lexicographic optimization leverage a dynamic gradient descent~\citep{gong2021bi}:
\begin{align}
    \theta\leftarrow\theta - \omega\cdot(\nabla \hat\gL_2(\theta) + \lambda\nabla\hat{\gL}_1'(\theta)),
\end{align}
where $\omega$ is predefined positive step size, and $\lambda$ is designed to satisfy the following desiderata:

(1) When the constraint is not satisfied (i.e., $\hat{\gL}_1'(\theta)>\hat\xi$), we should let the optimization focus on decreasing $\hat\gL_1'(\theta)$ to meet the constraint as fast as possible. At the same time, the learning objective $\\hatgL_2(\theta)$ should serve as the secondary objective and minimized to the degree that it does not hurt the descent of $\hat\gL_1'(\theta)$.

(2) When the constraint is satisfied (i.e., $\hat{\gL}_1'(\theta)\le\hat\xi$), we should focus on optimizing $\hat\gL_2(\theta)$. However, the increasing rate of $\hat\gL_1'(\theta)$ should in properly controlled so that $\theta$ stays inside or nearby the feasible set while we minimize $\hat\gL_2(\theta)$. 

Both properties can be satisfied if $\lambda$ is selected by the following optimization:
\begin{align}
    \lambda \leftarrow \argmin_\lambda\{\|\nabla\hat\gL_2(\theta)-\lambda\|~~~s.t.~~~\nabla\hat\gL_1'(\theta)^T\lambda\ge \phi(\theta)\},
    \label{eq:lbd}
\end{align}
where we want $\lambda$ to be as close to $\nabla\hat\gL_2(\theta)$ as much as possible, but subject to a lower bound on the inner product of $\nabla\hat\gL_1'(\theta)$ and $\lambda$ to make sure that the change of $\hat\gL_1'(\theta)$ is controlled by the location of $\theta$. The $\phi(\theta)$ is the dynamic barrier function which balance the loss minimization with constraint satisfaction by controlling the inner product between $\hat\gL_1'(\theta)$ and $\lambda$. To achieve the desiderata on $\lambda$, we should keep $\phi(\theta)$ to have the same sign as $\hat\gL_1'(\theta)-\hat\xi$ so that the constraint $\hat\gL_1'(\theta)\le\hat\xi$ is equivalent to $\{\theta: \phi(\theta)\le 0\}$, which is:
\begin{align}
    sign(\phi(\theta))=sign(\hat\gL_1'(\theta)-\hat\xi)
\end{align}
It is straight to see that the following choice of $\lambda$ satisfies the dual problem of Eq.~\ref{eq:lbd}:
\begin{align}
    \lambda =\argmin_{\lambda\ge 0}\{\|\nabla\gL_2(\theta)+\lambda \gL_1'(\theta)\|^2-\lambda\phi(\theta)\} = \max(\frac{\phi(\theta) - \nabla\hat\gL_2(\theta)^T\nabla\hat{\gL}_1'(\theta)}{\|\nabla\hat{\gL}_1'(\theta)\|^2}, 0)
\end{align}
with the choice of $\phi(\theta)$:
\begin{align}
    \phi(\theta) =&\min(\alpha(\hat{\gL}_1'(\theta)-\gamma\cdot\hat\xi), \beta\|\nabla\hat{\gL}_1'(\theta)\|^2),
\end{align}
where $\alpha$, $\beta$ and $\gamma$ are predefined positive hyperparameters. 

\section{Proof of Theorem 1}
First, we define sub-exponential random variable and then and elucidate its relationship with sub-Gaussian random variable. 

\label{app:bound}
\begin{definition}[Sub-exponential random variable]
    The random variable $X$ with mean $\mu$ is sub-exponential with parameters $(v, b)$ if for $\forall \lambda<\frac{1}{b}$, $\mathbb E_X[\exp\{\lambda(X-\mu)\}]\le \exp(\frac{v^2\lambda^2}{2})$.
\end{definition}

\begin{lemma}
    The square of a zero-mean sub-Gaussian random variable with parameter $\sigma^2$ is a sub-exponential random variable with parameter $(4\sqrt{2}\sigma^2, 4\sigma^2)$~\citep{honorio2014tight}.
    % \yub{can omit proof (this is standard, just cite a ref).}
\end{lemma}

\begin{proof}
    Given that $\meps_\theta(\rvx^{(t)}, t)$ and $\meps_\theta(\rvx^{(t)}, \rvc, t)$ are sub-Gaussian random variables parameterized by $\sigma^2$, and $\meps$ is the multivariate standard Gaussian random variable, then their subtraction, $\meps_\theta(\rvx^{(t)}, t) - \meps$ and $\meps_\theta(\rvx^{(t)}, \rvc, t)-\meps$, are still sub-Gaussian random variables parameterized by $\sigma^2+1$. Then $\hat{\gL}_1(\theta) = \mathbb E_{\rvx\sim \hat{p}_{\gD}(\rvx), t}[\vert\vert\meps_{\theta}(\rvx^{(t)}, t) - \meps\vert\vert^2]$ and $\hat{\gL}_2(\theta) = \mathbb E_{\rvx, \rvc\sim \hat{p}_{\gD_p}(\rvx, \rvc), t}\vert\vert\meps_{\theta}(\rvx^{(t)}, \rvc, t) - \meps\vert\vert^2$, the square of sub-Gaussian random variables, are sub-exponential random variables parameterized by $(4\sqrt{2}\tilde{\sigma}^2, 4\tilde{\sigma}^2)$, based on the Lemma above.
    
    Then, based on Bernstein’s inequality, we have:
    \begin{align}
        p(\vert\gL_2(\theta)-\hat{\gL}_2(\theta)\vert>\epsilon) \le 2\exp (-\frac{N_p\epsilon^2}{8\sqrt{2}\tilde{\sigma}^2}\wedge \frac{N_p\epsilon}{8\tilde{\sigma}^2})
    \end{align}
    Based on the union bound inequility, we further have:
    \begin{align}
    p(\sup_{\theta\in\Theta}\vert\gL_2(\theta)-\hat{\gL}_2(\theta)\vert>\epsilon)\le&\sum_{\theta\in\Theta}p(\vert\gL_2(\theta)-\hat{\gL}_2(\theta)\vert>\epsilon)\nonumber \\
    \le&2\vert\Theta\vert\exp (-\frac{N_p\epsilon^2}{8\sqrt{2}\tilde{\sigma}^2}\wedge\frac{N_p\epsilon}{8\tilde{\sigma}^2})
    \label{eq:l1}
    \end{align}
    Following the same way, we have:
    \begin{align}
    p(\sup_{\theta\in\Theta}\vert\gL_1'(\theta) - \hat{\gL}_1'(\theta)\vert>\epsilon)\le&2\vert\Theta\vert\exp (-\frac{N_p\epsilon^2}{8\sqrt{2}\tilde{\sigma}^2}\wedge\frac{N_p\epsilon}{8\tilde{\sigma}^2})
    \label{eq:l1p}
    \end{align}
    \begin{align}
        p(\sup_{\theta\in\Theta}\vert\gL_1(\theta) - \hat{\gL}_1(\theta)\vert>\epsilon)\le&2\vert\Theta\vert\exp (-\frac{N\epsilon^2}{8\sqrt{2}\tilde{\sigma}^2}\wedge \frac{N\epsilon}{8\tilde{\sigma}^2})
        \label{eq:l2ieq}
    \end{align}

Let the probability on the RHS of Eq. \ref{eq:l1}, Eq. \ref{eq:l1p} and Eq. \ref{eq:l2ieq} be $\delta$, then we compute $\epsilon$ and plug in into LHS, then with the probability of $1-\delta$ we have:
\begin{align}
    \sup_{\theta\in\Theta}\vert\gL_2(\theta) - \hat{\gL}_2(\theta)\vert\le \epsilon_{N_p}
    \label{eq:enp1}
\end{align}

\begin{align}
    \sup_{\theta\in\Theta}\vert\gL_1'(\theta) - \hat{\gL}_1'(\theta)\vert\le \epsilon_{N_p}
    \label{eq:enp2}
\end{align}

\begin{align}
    \sup_{\theta\in\Theta}\vert\gL_1(\theta) - \hat{\gL}_1(\theta)\vert\le \epsilon_N
    \label{eq:en}
\end{align}

where $\epsilon_N = \sqrt{8\sqrt{2}\tilde{\sigma}^2}\cdot\sqrt{\frac{\log\vert\Theta\vert+\log\frac{2}{\delta}}{N}}\vee8\tilde{\sigma}^2\cdot\frac{\log\vert\Theta\vert+\log\frac{2}{\delta}}{N}$ and $\epsilon_{N_p} = \sqrt{8\sqrt{2}\tilde{\sigma}^2}\cdot\sqrt{\frac{\log\vert\Theta\vert+\log\frac{2}{\delta}}{N_p}}\vee8\tilde{\sigma}^2\cdot\frac{\log\vert\Theta\vert+\log\frac{2}{\delta}}{N_p}$ are used to simplify the notation.

Let $\epsilon = \epsilon_N + \epsilon_{N_p}$. From Eq. \ref{eq:en}, we have:
\begin{align}
    & \vert\hat{\xi} - \xi\vert\le \epsilon_N \\
    \Longrightarrow& \xi \le \hat{\xi} + \epsilon_N \\
    \Longrightarrow& \xi+\epsilon_{N_p} \le \hat{\xi} + \epsilon_{N_p} + \epsilon_N
    \label{eq:xi}
\end{align}
Based on Eq. \ref{eq:enp2}, similarly, we have $\hat{\gL}_1'(\theta)\le\gL_1'(\theta)+\epsilon_{N_p}$. Since $\xi=\inf_{\theta\in\Theta}\gL_1(\theta)$, then for $\forall \theta^*$ s.t. $\gL_1'(\theta^*)=\xi$, according to Eq. \ref{eq:xi}, we can obtain:
\begin{align}
\hat{\gL}_1'(\theta)\le\gL_1'(\theta)+\epsilon_{N_p}\le\xi+\epsilon_{N_p}\le\hat{\xi}+\epsilon_{N_p} +\epsilon_N.
\label{eq:l1p_xi}
\end{align}
Let $\hat{\theta}^*$ be the solution of Eq. \ref{eq:emp_l2} and $\theta^*$ be the solution of Eq. \ref{eq:obj_l2}. Additionally, let $\Theta^*=\{\theta:\gL_1'(\theta)\le\xi\}$, and $\hat{\Theta}^*=\{\theta:\hat{\gL}_1'(\theta;\rvx)\le\hat{\xi}+\epsilon\}$. We state that $\Theta^*\subseteq\hat{\Theta}^*$ based on Eq. \ref{eq:l1p_xi}.

Next, we prove that $\hat{\theta}^*$ competes well with $\theta^*$ on $\hat{\gL}_2(\theta; \rvx, \rvc)$ and $\gL_2(\theta)$:
\begin{align}
    \gL_2(\hat{\theta}^*)\le&\hat{\gL}_2(\hat{\theta}^*;\rvx,\rvc) + \epsilon_{N_p} & \text{based on Eq. \ref{eq:enp1}} \nonumber \\
    =&\min_{\theta\in\hat{\Theta}^*}\hat{\gL}_2(\theta;\rvx,\rvc) + \epsilon_{N_p} & \text{$\hat{\theta}$ is the solution of Eq. \ref{eq:emp_l2}} \nonumber \\
    \le&\min_{\theta\in\Theta^*}\hat{\gL}_2(\theta;\rvx,\rvc) + \epsilon_{N_p} & \text{$\Theta^*\subseteq\hat{\Theta}^*$}\nonumber \\
    \le&\min_{\theta\in\Theta^*}\gL_2(\theta) + 2\epsilon_{N_p} & \text{based on Eq. \ref{eq:enp1}}\nonumber \\
    =&\gL_2(\theta^*) + 2\epsilon_{N_p}
\end{align}

Next, we prove that $\hat{\theta}^*$ does not violate constraint too much:
\begin{align}
    \gL_1'(\hat{\theta}^*)\le & \hat{\gL}_1'(\hat{\theta}^*;\rvx)&\text{based on Eq. \ref{eq:enp2}} \nonumber \\
    \le & \hat{\xi} + \epsilon +\epsilon_{N_p} & \text{based on definition of $\hat{\Theta}^*$} \nonumber\\
    =& \hat{\xi} + 2\epsilon_{N_p} + \epsilon_N & \text{$\epsilon = \epsilon_{N_P} + \epsilon_N$} \nonumber\\
    \le& \xi + 2\epsilon_{N_p} + 2\epsilon_N & \text{based on Eq. \ref{eq:xi}}
\end{align}
\end{proof}

\section{More details regarding dataset}
\label{sec:text}

\subsection{Stocks employed to form time series data}
We construct the dataset by assembling 24 stocks from Yahoo Finance during their IPO date to July 8, 2023, including Ethereum USD, NVIDIA Corporation, AT\&T, Accenture plc, The Boeing Company, The Coca-Cola Company, Simon Property Group Inc., NIKE Inc., Sony Group Corporation, Chegg Inc., UnitedHealth Group Incorporated, General Motors Company, Russell 2000, JPMorgan Chase \& Co., Salesforce Inc., Lockheed Martin Corporation, Walmart Inc., NASDAQ Composite, Shell plc, Pfizer Inc., Bitcoin USD, Apple Inc., Amazon.com Inc., Alphabet Inc.
\subsection{Splitting training and testing sets}
Each dataset is divided into training and testing sets at the ratio of $80\%$ and $20\%$, respectively. We curate each dataset to have varying proportions (i.e., $2\%, 4\%, 6\%, 8\%, 10\%, 20\%, 30\%, 40\%$) of text labels.
\subsection{Constructing text descriptions}
We construct the text descriptions of molecules by two templates: (1) exact description, e.g., ``A molecule with the heat capacity of -0.11, the lomo of 0.87, the homo of -0.21, the polarizability of 0.95, the dipole moment of -1.61 and the energy gap between homo and lumo as 0.94.", and (2) general description, e.g., ``A molecule with a high homo value, a very low heat capacity, a medium polarizability, a high energy difference between homo and lomo, a very high lomo value and a high dipole moment.". These text descriptions are then refined by GPT 3.5 and the order of different properties is shuffled. We form the text description for time series data using three templates: (a) exact description, e.g., ``A time series with the frequency of 0.017, the mean of 3.12e-05, 19 peaks, the variance of 1.18e-11, the linear trend of 0.12 and the skewness of -6.15."; (b) general description, e.g., ``A time series with large average, medium frequency, nearly equal large and small values, medium negative linearity, a few peaks and large variance" and (c) description of trend, e.g., ``A time series that first stays stable, then increases with the slope of 1.25". Then those descriptions are further refined by GPT-3.5.
\section{Implementation details}
\label{app:imp}
\subsection{Details of training process}
We leverage a cross-attention layer to enhance the alignment between data and text embeddings~\citep{ma2023glyphdraw}. Specifically, we maximize the cosine-similarity between text embeddings obtained from the pretrained LLMs encoder and data embeddings diffused from original data, where the last layer of LLMs encoder is finetuned during training. Baselines models are trained for 3,000 epoches. Their finetuned version and Text2Data are both pretrained for 1,000 epoches and finetuned for another 2,000 epoches. All experiments are conduced on A100 hardware. 

For data augmentation as an ablation study (i.e., EDM-DA), we employed GPT-4 to modify the textual descriptions to augment the text-data pairs.
\subsection{Score functions}
For time series generation, we use the same score function as in \cite{rasul2021autoregressive} that consists of conditional 1-dim dilated ConvNets with residual connections adapted from the WaveNet~\citep{oord2016wavenet} and Diffwave~\citep{kong2020diffwave}. For molecule generation, we employ the same score function from EDM~\cite{hoogeboom2022equivariant} which is a Multilayer Perceptron. For motion generation, we also leverage the same score function in \cite{tevet2023human} as a straightforward transformer.

\section{Additional results}
Table~\ref{tab:mol_prop} shows that the MAE of predicted properties against intended properties of generated molecules from 2$\%$ to 100$\%$ available labels. Table~\ref{tab:mol_ablation} shows that two additional ablation studies on Molecules dataset evaluating the controllability of Text2Data. EDM-finetune-unlabel is pretrained on only unlabeled data and finetuned on labeled data; EDM-DA is trained on text-data paired augmented by GPT-4.  EDM-semi is the semi-supervised ablation studies we have introduced above. EDM-finetune is pretrained on all data (i.e., molecules with and without labels) while EDM-finetune-unlabel is pretrained on only unlabeled data and finetuned on labeled data. EDM-DA is trained on text-data paired augmented by GPT-4. EDM is a classifier-free diffusion model trained for controllable molecule generation that we choose as the baseline. The results suggest that the performance of EDM-semi is 54.65$\%$ worse than Text2Data on average, and it's even worse than the EDM baseline. This is because that EDM-semi heavily relies on the precision of the classifier to generate pseudo labels, which is usually hard given complex molecular structures. Bad pseudo label mislead the alignment between data and text during finetuning.
\begin{table*}[hbt!]
 \caption{Evaluate controllability on Molecules dataset by MAE according to $\alpha$ and $\epsilon_{HOMO}$ of molecules.}
\centering
\begin{adjustbox}{width=\textwidth,center}
\tiny
\begin{tabular}{c c ccc c ccc} 
\hline
\multirow{2}{*}{Proportion (\%)} && \multicolumn{3}{c}{$\alpha$} &&
\multicolumn{3}{c}{$\epsilon_{HOMO}$} \\\cline{3-5} \cline{7-9}
&& Text2Data & EDM-finetune & EDM && Text2Data & EDM-finetune & EDM\\ \hline
2 &&\textbf{0.66} &0.84 &0.83 &&\textbf{0.58} &0.92 &1.00 \\
4 &&\textbf{0.62} &0.74 &0.69 &&\textbf{0.56} &0.84 &1.07 \\
6 &&\textbf{0.44} &0.68 &0.69 &&\textbf{0.44} &0.92 &1.00 \\
8 &&\textbf{0.52} &0.67 &0.69 &&\textbf{0.47} &0.90 &0.95 \\
10 &&\textbf{0.34} &0.66 &0.75 &&\textbf{0.33} &0.89 &0.92 \\
20 &&\textbf{0.41} &0.60 &0.67 &&\textbf{0.42} &0.82 &0.91 \\
40 &&\textbf{0.43} &0.55 &0.55 &&\textbf{0.42} &0.78 &0.81 \\
60 &&\textbf{0.41} &0.50 &0.52 &&\textbf{0.43} &0.61 &0.59 \\
80 &&\textbf{0.40} &0.41 &0.43 &&\textbf{0.42} &0.50 &0.51 \\
100 &&\textbf{0.40} &0.40 &0.40 &&\textbf{0.42} &0.43 &0.42 \\
\hline
\end{tabular}
\end{adjustbox}
\label{tab:mol_prop}
\end{table*}

\begin{table*}[hbt!]
 \caption{Evaluate controllability on Molecules dataset by MAE according to $\alpha$ and $\epsilon_{HOMO}$ of molecules. Two additional ablation studies are added: EDM-finetune-unlabel is pretrained on only unlabeled data and finetuned on labeled data; EDM-DA is trained on text-data paired augmented by GPT-4.}
\centering
\begin{adjustbox}{width=\textwidth,center}
\begin{tabular}{c c cccccc c ccccc} 
\hline
\multirow{2}{*}{Proportion (\%)} && \multicolumn{6}{c}{$\alpha$} &&
\multicolumn{5}{c}{$\epsilon_{HOMO}$} \\\cline{3-8} \cline{10-14}
&& Text2Data & EDM-finetune &EDM-finetune-unlabel & EDM-DA & EDM-semi & EDM && Text2Data & EDM-finetune &EDM-finetune-unlabel & EDM-DA & EDM\\ \hline
2 &&\textbf{0.66} &0.84 &0.67 &0.80 &0.87 &0.83 &&\textbf{0.58} &0.92 &0.58 &0.95 &1.00 \\
4 &&\textbf{0.62} &0.74 &0.65 &0.69 &0.86 &0.69 &&\textbf{0.56} &0.84 &0.58 &0.93 &1.07 \\
6 &&\textbf{0.44} &0.68 &0.56 &0.67 &0.78 &0.69 &&\textbf{0.44} &0.92 &0.50 &0.90 &1.00 \\
8 &&\textbf{0.52} &0.67 &0.54 &0.66 &0.74 &0.69 &&\textbf{0.47} &0.90 &0.50 &0.88 &0.95 \\
10 &&\textbf{0.34} &0.66 &0.46 &0.61 &0.74 &0.75 &&\textbf{0.33} &0.89 &0.43 &0.88 &0.92 \\
\hline
\end{tabular}
\end{adjustbox}
\label{tab:mol_ablation}
\end{table*}

The results of evaluating the controllability of Text2Data and its baseline comparisons regarding variance, number of peaks and linearity of generated time series are in Table~\ref{tab:ts_cond2}.

\begin{table}[hbt!]
 \caption{Evaluate controllability on time series by MAE on testing set, according to different proportions of paired data. Lower MAE indicates better performance.}
\centering
\begin{adjustbox}{width=\columnwidth,center}
\Large
\begin{tabular}{c c ccc c ccc c ccc } 
\hline\hline
\multirow{2}{*}{Proportion (\%)} && \multicolumn{3}{c}{Variance ($\times 10^{-5}$)} &&\multicolumn{3}{c}{Number of peaks} &&\multicolumn{3}{c}{Linearity} \\\cline{3-5} \cline{7-9} \cline{11-13} 
&& Text2Data & DiffTS-finetune & DiffTS && Text2Data & DiffTS-finetune & DiffTS && Text2Data & DiffTS-finetune & DiffTS \\ \hline\hline
2 &&\textbf{5.28$\pm$11.40} &5.35$\pm$11.20 &5.70$\pm$11.40 &&\textbf{12.94$\pm$0.88} &12.95$\pm$0.87 & 13.01$\pm$0.83&&0.61$\pm$0.04 &\textbf{0.61$\pm$0.04} & 0.62$\pm$0.04\\
4 &&\textbf{5.37$\pm$11.28} &5.48$\pm$11.24 &6.08$\pm$11.31&&12.91$\pm$0.84 &\textbf{12.90$\pm$0.86} &12.93$\pm$0.84 &&\textbf{0.61$\pm$0.04} &0.61$\pm$0.04 & 0.62$\pm$0.04\\ 
6 &&\textbf{5.04$\pm$10.90} &5.19$\pm$10.42 &5.60$\pm$10.20 &&\textbf{12.89$\pm$0.90} &12.90$\pm$0.83 & 12.90$\pm$0.83&&\textbf{0.61$\pm$0.04} &0.61$\pm$0.04 & 0.62$\pm$0.04\\ 
8 &&\textbf{5.30$\pm$10.95} &5.59$\pm$11.19 &5.60$\pm$10.40 &&\textbf{12.88$\pm$0.86} &12.89$\pm$0.85 &12.95$\pm$0.87 &&\textbf{0.61$\pm$0.04} &0.61$\pm$0.04 & 0.61$\pm$0.04\\
10 &&\textbf{5.08$\pm$20.48} &5.41$\pm$16.23 &5.80$\pm$11.00 &&\textbf{12.75$\pm$0.93} &12.87$\pm$0.90 &12.90$\pm$0.83 &&\textbf{0.61$\pm$0.04} &0.61$\pm$0.04 & 0.61$\pm$0.04\\
20 &&\textbf{5.09$\pm$10.98} &5.37$\pm$11.09 &6.40$\pm$11.20 &&\textbf{12.87$\pm$0.85} & 12.88$\pm$0.85&12.90$\pm$0.84 &&\textbf{0.61$\pm$0.04} &0.61$\pm$0.04 &0.61$\pm$0.04 \\
30 &&\textbf{4.85$\pm$10.58} &5.19$\pm$9.90 &5.70$\pm$10.20 &&\textbf{12.84$\pm$0.83} &12.87$\pm$0.87 &12.91$\pm$0.88 &&\textbf{0.61$\pm$0.04} &0.61$\pm$0.04 & 0.61$\pm$0.04\\ 
40 &&\textbf{4.77$\pm$10.42} &5.11$\pm$14.10 &5.34$\pm$13.10 &&\textbf{12.84$\pm$0.87} &12.88$\pm$0.90 &12.89$\pm$0.86 &&\textbf{0.60$\pm$0.04} &0.61$\pm$0.04 & 0.61$\pm$0.04\\ 
\hline
\end{tabular}
\end{adjustbox}
\label{tab:ts_cond2}
\end{table}
\begin{table*}[hbt!]
 \caption{Evaluate generation quality on Molecule dataset by $-\log p$ and validity according to different proportions of paired data. Lower $-\log p$ and higher validity indicate better performance.}
\centering
\begin{adjustbox}{width=\textwidth,center}
\tiny
\begin{tabular}{c c ccc c ccc} 
\hline
\multirow{2}{*}{Proportion (\%)} && \multicolumn{3}{c}{-log p $\downarrow$} &&
\multicolumn{3}{c}{Validity $\uparrow$} \\\cline{3-5} \cline{7-9}
&& Text2Data & EDM-finetune & EDM && Text2Data & EDM-finetune & EDM\\ \hline
2 && \textbf{-111.39$\pm$0.92}&-74.88$\pm$1.82 &-49.15$\pm$0.96 && \textbf{0.97$\pm$0.07}&0.93$\pm$0.11 & 0.86$\pm$0.09\\
4 && \textbf{-119.09$\pm$0.30}&-87.56$\pm$4.18 &-78.72$\pm$2.95 &&0.95$\pm$0.07 &\textbf{0.96$\pm$0.07} &0.83$\pm$0.15 \\
% 6 && \textbf{-127.6850} &-82.5125 &-96.3332 && \textbf{1.0000} &0.7813 &0.6250 && \textbf{1.0000} &0.9853 &0.9657 && & & && & & \\
6 && \textbf{-119.55$\pm$0.59}&-97.32$\pm$2.09 &-69.58$\pm$1.97 && \textbf{0.97$\pm$0.07}& 0.96$\pm$0.05& 0.81$\pm$0.14\\
8 && \textbf{-119.40$\pm$0.68} & -101.31$\pm$1.11& -85.19$\pm$1.58&& \textbf{0.97$\pm$0.07}& 0.93$\pm$0.07& 0.90$\pm$0.11\\
10 && \textbf{-121.37$\pm$1.24}& -104.17$\pm$1.94&-85.73$\pm$1.00 && \textbf{0.96$\pm$0.06}& 0.95$\pm$0.13& 0.88$\pm$0.10\\
20 &&\textbf{-119.58$\pm$1.61} & -104.08$\pm$2.03&-76.39$\pm$2.11 && \textbf{0.97$\pm$0.07}&0.95$\pm$0.07 &0.90$\pm$0.09 \\
30 &&\textbf{-121.00$\pm$1.07} & -115.58$\pm$0.95&-76.22$\pm$1.26 && \textbf{0.97$\pm$0.07}&0.95$\pm$0.07 &0.91$\pm$0.09 \\
40 && \textbf{-119.90$\pm$0.97}&-114.00$\pm$0.59 & -80.97$\pm$0.09&& \textbf{0.97$\pm$0.07}& 0.95$\pm$0.08& 0.90$\pm$0.10\\
\hline
\end{tabular}
\end{adjustbox}
\label{tab:mol_gen1}
\end{table*}

\begin{table*}[hbt!]
\caption{Evaluate generation quality on Molecule dataset by molecular and atom stability of generated molecules according to different proportions of paired data. Higher molecular stability and atom stability indicate better performance.}
\centering
\begin{adjustbox}{width=\textwidth,center}
\tiny
\begin{tabular}{c c ccc c ccc} 
\hline
\multirow{2}{*}{Proportion (\%)} && \multicolumn{3}{c}{Mol. Stability $\uparrow$} &&
\multicolumn{3}{c}{Atom Stability $\uparrow$} \\\cline{3-5} \cline{7-9}
&& Text2Data & EDM-finetune & EDM && Text2Data & EDM-finetune & EDM\\ \hline
2 && \textbf{0.86$\pm$0.14}& 0.85$\pm$0.04 &0.65$\pm$0.04 &&\textbf{0.99$\pm$0.01} &\textbf{0.99$\pm$0.01} &0.96$\pm$0.01 \\
4 && \textbf{0.87$\pm$0.09}& 0.83$\pm$0.08&0.66$\pm$0.09 &&\textbf{0.99$\pm$0.01} &0.98$\pm$0.01 &0.97$\pm$0.01 \\
% 6 && \textbf{-127.6850} &-82.5125 &-96.3332 && \textbf{1.0000} &0.7813 &0.6250 && \textbf{1.0000} &0.9853 &0.9657 && & & && & & \\
6 &&\textbf{0.88$\pm$0.16} &0.86$\pm$0.06 &0.73$\pm$0.04 &&\textbf{0.99$\pm$0.02} &\textbf{0.99$\pm$0.01} &0.97$\pm$0.01 \\
8 && \textbf{0.88$\pm$0.11}&0.87$\pm$0.04 &0.79$\pm$0.05 && \textbf{0.99$\pm$0.01}&\textbf{0.99$\pm$0.01} &0.98$\pm$0.01 \\
10 && \textbf{0.88$\pm$0.10}&0.83$\pm$0.06 &0.77$\pm$0.11 && \textbf{0.99$\pm$0.01}&0.98$\pm$0.01 &0.98$\pm$0.01 \\
20 && \textbf{0.88$\pm$0.10}& 0.87$\pm$0.08& 0.79$\pm$0.09&& \textbf{0.99$\pm$0.01}&0.98$\pm$0.01 &0.98$\pm$0.01 \\
30 &&\textbf{0.89$\pm$0.10} &0.87$\pm$0.07 &0.79$\pm$0.09 &&\textbf{0.99$\pm$0.01} & 0.98$\pm$0.01&0.98$\pm$0.01 \\
40 &&\textbf{0.89$\pm$0.10} &0.86$\pm$0.06 &0.79$\pm$0.09 && \textbf{0.99$\pm$0.01}&0.98$\pm$0.01 &0.98$\pm$0.01 \\
\hline
\end{tabular}
\end{adjustbox}
\label{tab:mol_gen2}
\end{table*}
\begin{table*}[hbt!]
\caption{Evaluate generation quality on HumanML3D dataset by FID and Diversity according to different proportions of paired data. Low FID and higher diversity indicate better performance.}
\centering
\begin{adjustbox}{width=\textwidth,center}
\tiny
\begin{tabular}{c c ccc c ccc} 
\hline
\multirow{2}{*}{Proportion (\%)} && \multicolumn{3}{c}{FID $\downarrow$} &&
\multicolumn{3}{c}{Diversity $\uparrow$} \\\cline{3-5} \cline{7-9}
&& Text2Data & MDM-finetune & MDM && Text2Data & MDM-finetune & MDM  \\ \hline
2 &&\textbf{1.22$\pm$0.12} &1.23$\pm$0.03 &3.09$\pm$0.12 &&\textbf{9.26$\pm$0.07} & 9.08$\pm$0.17&8.84$\pm$0.14 \\
4 && \textbf{1.12$\pm$0.11}&1.16$\pm$0.17 &3.06$\pm$0.21 &&\textbf{9.31$\pm$0.21} &9.13$\pm$0.20 &8.80$\pm$0.18 \\
6 &&1.00$\pm$0.13 &\textbf{0.64$\pm$0.11} &1.13$\pm$0.12 &&\textbf{9.40$\pm$0.13} & 9.39$\pm$0.22&8.91$\pm$0.22 \\
8 && 1.40$\pm$0.14& \textbf{1.32$\pm$0.15}& 1.34$\pm$0.18&& \textbf{9.43$\pm$0.20}& 9.21$\pm$0.18&8.95$\pm$0.19 \\
10 &&\textbf{1.49$\pm$0.19} &1.52$\pm$0.13 & 1.50$\pm$0.15&&9.59$\pm$0.09 &\textbf{9.74$\pm$0.13} & 9.37$\pm$0.15\\
20 && \textbf{0.92$\pm$0.06}&1.02$\pm$0.12 &1.07$\pm$0.06 && \textbf{9.77$\pm$0.20}& 9.72$\pm$0.17&9.66$\pm$0.15 \\
30 &&\textbf{0.81$\pm$0.10} &0.99$\pm$0.13 &1.11$\pm$0.10 &&\textbf{9.79$\pm$0.11} &9.70$\pm$0.12 &9.63$\pm$0.14 \\
40 && \textbf{0.63$\pm$0.12}& 0.95$\pm$0.11&1.13$\pm$0.15 && \textbf{9.74$\pm$0.13}& 9.70$\pm$0.20& 9.38$\pm$0.15\\
\hline
\end{tabular}
\end{adjustbox}
\label{tab:mt_gen}
\end{table*}
Table~\ref{tab:mol_gen1} shows the results of evaluating the generation quality of Text2Data and baseline comparisons on Molecule dataset based on $-\log p$ and atom validity. 

Table~\ref{tab:mol_gen2} shows the results of evaluating the generation quality of Text2Data and baseline comparisons on Molecule dataset based on molecular and atom stability.

Table~\ref{tab:mt_gen} shows the results of evaluating the generation quality of Text2Data and baseline comparisons on motion dataset based on FID and diversity.

Figure~\ref{fig:tsne} shows the t-SNE plot on time series data generated by Text2Data, DiffTS-finetune model and DiffTS, where Text2Data generated time series are more aligned with the ground truth indicated by their overlap.

\begin{figure*}[hbt!]
\centering
\begin{subfigure}{.26\textwidth}
  \includegraphics[width=\linewidth]{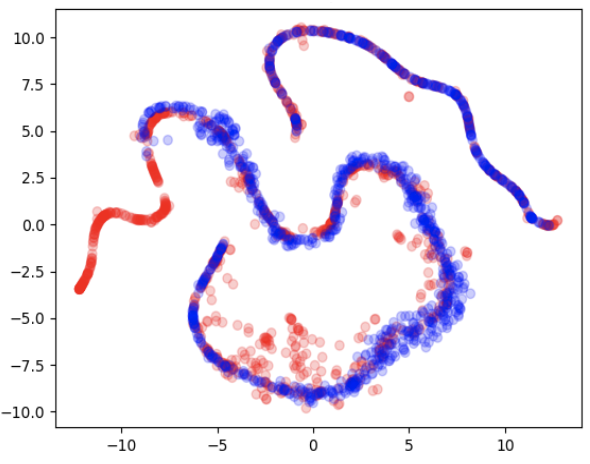}  
  \caption{Text2Data}
\end{subfigure}
\begin{subfigure}{.26\textwidth}
  \includegraphics[width=\linewidth]{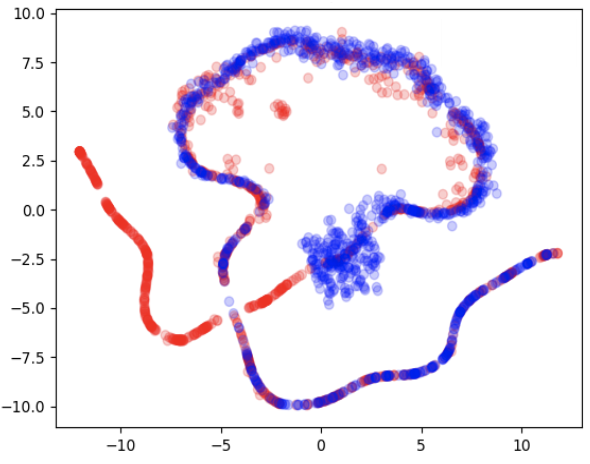}
  \caption{DiffTS-finetune}
\end{subfigure}
\begin{subfigure}{.26\textwidth}
  \includegraphics[width=\linewidth]{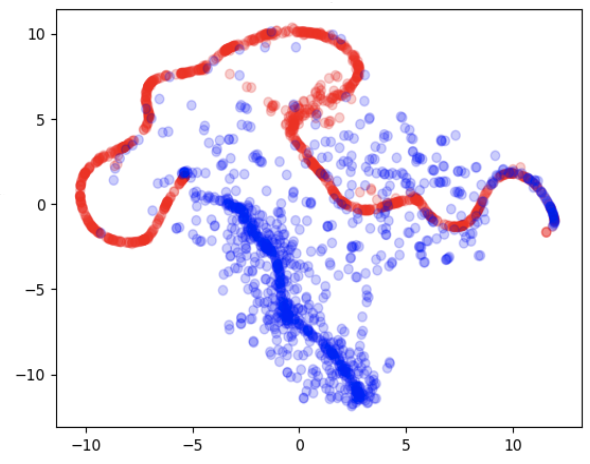}  
  \caption{DiffTS}
\end{subfigure}
\caption{t-SNE visualization on time series data generated by Text2Data, DiffTS-finetune model and DiffTS. Red denotes ground truth, and blue denotes generated data.}
\label{fig:tsne}
\end{figure*}
% \nobibliography*
\section{Ethical statement}
We develop our method from publicly available QM9~\citep{ramakrishnan2014quantum} and HumanML3D~\citep{Guo_2022_CVPR} datasets, and stock data from public Yahoo Finance. It is important to note that, like other text-to-data models, our implementation will likely reflect the socioeconomic and entity biases inherent in datasets that we use. Additionally, although our method is designed for controllable data generation from text, we are not able to control the prompt that user inputs, which may contain improper contents.
\bibliography{aaai25}

\end{document}